\documentclass{article} 
\usepackage[final]{colm2026_conference}

\usepackage{microtype}
\usepackage{hyperref}
\usepackage{url}
\usepackage{booktabs}

\usepackage{times}
\usepackage{latexsym}
\usepackage[T1]{fontenc}
\usepackage[utf8]{inputenc}
\usepackage{microtype}
\usepackage{inconsolata}

\usepackage{amsmath}
\usepackage{amssymb}
\usepackage{amsfonts}
\usepackage{mathtools}
\usepackage{amsthm}
\usepackage{nicefrac}

\usepackage{booktabs}
\usepackage{multirow}
\usepackage{threeparttable}
\usepackage{tabularx}
\usepackage{colortbl}

\usepackage{graphicx}
\usepackage{subcaption}
\usepackage{wrapfig}
\usepackage{xcolor}

\usepackage{tikz}
\usepackage{pgfplots}
\usetikzlibrary{intersections}
\usepgfplotslibrary{fillbetween,groupplots,dateplot}
\pgfplotsset{compat=newest, scaled y ticks=false}

\usepackage{algorithm}
\usepackage{algpseudocode}

\usepackage{url}
\usepackage{bookmark}
\usepackage[capitalize,noabbrev]{cleveref}

\usepackage{pifont}

\usepackage{xspace}

\theoremstyle{plain}
\newtheorem{theorem}{Theorem}[section]
\newtheorem{proposition}[theorem]{Proposition}

\theoremstyle{definition}

\theoremstyle{remark}
\newtheorem{remark}[theorem]{Remark}

\usepackage{etoc}
\usepackage{newtxtext}

\newcommand{\Mh}{\mathbf{M}_h}
\newcommand{\Wq}{\mathbf{W}_q}
\newcommand{\Wk}{\mathbf{W}_k}
\newcommand{\Method}{QK Product Steering}

\usepackage{lineno}

\definecolor{darkblue}{rgb}{0, 0, 0.5}
\hypersetup{colorlinks=true, citecolor=darkblue, linkcolor=darkblue, urlcolor=darkblue}

\title{Spectral Query-Key Product Weight Steering for Training-Free VLM Hallucination Mitigation}

\author{
Karn Tiwari$^{1}$ \quad
Varnith Chordia$^{2}$ \quad
Prathosh A P$^{1, 3}$ \\
$^{1}$Indian Institute of Science, Bengaluru \\
$^{2}$Snap Inc. \\
$^{3}$LatentForce.ai \\
\texttt{\{karntiwari, prathosh\}@iisc.ac.in} \quad
\texttt{vchordia@snapchat.com}
}

\begin{document}

\ifcolmsubmission
\linenumbers
\fi

\maketitle

\begin{abstract}
Vision-language models (VLMs) often generate fluent but visually unsupported descriptions, especially by mentioning objects absent from the image. We propose \emph{QK Product Steering}, a data-free, training-free, and zero-inference-cost weight edit for reducing object hallucination. The method directly edits the per-head query-key product, the operator that produces pre-softmax attention logits, by suppressing a small number of dominant singular modes in selected middle layers. The edited product is then mapped back to the query weights through a closed-form query-only update while keeping shared key weights fixed, making the edit compatible with grouped-query attention. We further decompose the QK product into symmetric and antisymmetric components to distinguish mutual content-similarity patterns from directional attention patterns. Across three GQA-based VLMs, QK Product Steering achieves an average relative CHAIR$_s$ reduction of $4.0\%$, while matched random-mode controls show negligible change. Interpretability ablations show that the hallucination signal is specific to dominant QK modes and is primarily localized to the symmetric mutual-attention channel. Overall, QK Product Steering offers a simple alternative to decoding-time mitigation, requiring no additional data, fine-tuning, or inference-time overhead while largely preserving general multimodal capability.
\end{abstract}

\etocdepthtag.toc{mtmain}

\section{Introduction}
\label{sec:intro}
\vspace{-0.2cm}
Vision-language models (VLMs) have become increasingly capable of generating fluent descriptions and answers from visual inputs. However, they still suffer from \emph{object hallucination}: generating objects that are not present in the image. This failure is especially problematic because hallucinated descriptions are often coherent and confident, making them difficult to detect. Reducing object hallucination is therefore essential for building reliable and trustworthy VLMs.

A key source of this problem is the competition between visual evidence and learned vision-language priors. During generation, a VLM must decide whether to rely on the image or on object co-occurrence patterns learned during large-scale pretraining. When these priors dominate, the model may mention plausible objects that fit the textual context but are unsupported by the visual input. Effective hallucination mitigation should therefore weaken such prior-driven behavior while preserving the model's general multimodal capability.

Existing methods address this problem mainly through additional training or inference-time intervention. Training-based approaches require extra data, supervision, and compute, making them expensive to apply across models. Training-free decoding methods avoid fine-tuning, but they alter generation through contrastive decoding, logit adjustment, attention reweighting, or penalty terms, introducing additional cost at every decoding step. A more practical alternative is a one-shot weight edit: modify the model once and then keep inference unchanged.

In this work, we propose QK Product Steering, a data-free, training-free, and zero-inference-cost weight-editing method for mitigating object hallucination in VLMs. Our key idea is to edit the query-key product, the operator that directly determines pre-softmax attention logits. Instead of modifying decoding, QK Product Steering identifies dominant singular modes of the per-head query-key product in selected middle layers and suppresses a small number of them. The edited product is then projected back into the query weights through a closed-form query-only update while keeping the shared key weights unchanged, making the method compatible with grouped-query attention.

The motivation behind QK Product Steering is that dominant query-key modes may capture high-energy attention directions that behave like generic vision-language priors. These directions support fluent generation, but they can also overpower image evidence and produce visually unsupported objects. By damping only a small number of dominant modes in middle layers, QK Product Steering weakens these overactive priors while preserving most of the model's original computation.

Beyond the main spectral edit, we further decompose the query-key product into symmetric and antisymmetric components. This allows us to separate \emph{mutual} content-similarity patterns from \emph{directional} attention patterns. Our analysis shows that the hallucination signal is primarily concentrated in the symmetric mutual-attention channel, while the antisymmetric channel carries little detectable hallucination effect. This provides a mechanistic explanation for why suppressing dominant query-key modes reduces hallucination: the edit weakens strong mutual vision-language priors rather than arbitrarily perturbing the model. Our contributions are summarized as follows:
\begin{itemize}
    \vspace{-0.2cm}
    \item We introduce \emph{QK Product Steering}, a data-free, training-free, and zero-inference-cost method for mitigating object hallucination in VLMs via a one-shot query-key product weight edit.
    \item We propose a GQA-safe query-only projection that realizes the edited query-key product while keeping shared key weights unchanged, preserving the standard inference procedure.
    \item We provide controlled ablations and symmetric/antisymmetric decomposition analysis showing that hallucination is specific to dominant QK modes and is primarily localized to the symmetric mutual-attention channel.
    \vspace{-0.2cm}
\end{itemize}

\vspace{-0.2cm}
\section{Related Work}
\label{sec:related}
\vspace{-0.2cm}

\textbf{Large Vision-Language Models.}
Large vision-language models (LVLMs) combine visual encoders with language models through cross-attention, query transformers, or multimodal instruction tuning, enabling open-ended captioning and visual question answering~\cite{alayrac2022flamingo,li2023blip2,dai2023instructblip,liu2023llava,bai2023qwen}. Recent models such as Qwen2.5-VL, InternVL3, and Pixtral-12B improve high-resolution perception and multimodal reasoning~\cite{qwen25vl,internvl3,pixtral2024}, but their generative flexibility also makes them prone to visually unsupported outputs.

\noindent\textbf{Object Hallucination Evaluation.}
Object hallucination refers to generating objects not grounded in the image. CHAIR evaluates hallucinated COCO object mentions in captions~\cite{rohrbach2018chair,lin2014coco}, while POPE probes object existence through yes/no questions and highlights the role of object-frequency and co-occurrence priors~\cite{li2023pope}. HallusionBench further tests visual-dependent and counterfactual reasoning~\cite{guan2023hallusion}, and benchmarks such as AMBER, MMHal-Bench, and M-HalDetect extend evaluation to broader factuality settings~\cite{wang2023amber,sun2024mmhal,gunjal2023mhaldetect}. We use CHAIR for long-form caption hallucination and report POPE, HallusionBench, MME, and MMMU to assess grounding and general capability.

\noindent\textbf{Hallucination Mitigation.}
Training-based methods reduce hallucination through instruction tuning, preference optimization, or correction data~\cite{liu2023lrv,xiao2024hsa,yue2024less,sun2024mmhal}. Training-free methods avoid finetuning but usually modify inference: VCD contrasts logits from clean and perturbed images~\cite{leng2024vcd}, OPERA penalizes over-trust during beam search~\cite{huang2024opera}, PAI increases image-token influence~\cite{liu2024paying}, Woodpecker verifies and corrects generated claims~\cite{yin2023woodpecker}, and Summary-Guided Decoding reduces language-prior accumulation~\cite{min2025sumgd}. In contrast, QK Product Steering edits the model once and leaves the standard decoding path unchanged.

\noindent\textbf{Steering and Mechanistic Interventions.}
Recent work localizes hallucination to activations, attention heads, and modality-specific subspaces. VTI, Activation Steering Decoding, VISTA, and Dynamic Multimodal Activation Steering intervene in hidden states or truthfulness directions~\cite{liu2025vti,su2025activation,li2025vista,yin2026dmas}. Attention-centered studies identify hallucination-relevant heads and layers, with middle layers playing an important role in visual information propagation~\cite{yang2025modular,jiang2025devils,zhu2025attentioncalibration}; ASCD steers attention scores during contrastive decoding~\cite{wang2025ascd}. Our method instead edits the query-key product itself, the bilinear operator that produces attention logits.

\noindent\textbf{Weight Editing and Grouped-Query Attention.}
Modern decoder backbones often use grouped-query attention (GQA), where multiple query heads share key/value heads to reduce KV-cache cost~\cite{ainslie2023gqa}. This sharing makes direct key-weight editing risky, since changing a shared key can unintentionally affect several query heads. Prior weight-editing methods such as Nullu reduce hallucination without inference overhead, but rely on prompt-derived hallucination subspaces~\cite{yang2025nullu}. In contrast, QK Product Steering is data-free: it edits dominant spectral modes of each head's query-key product and realizes the edit through a query-only update, preserving shared keys and standard inference.

\vspace{-0.2cm}
\section{Background and Preliminaries}
\label{sec:preliminaries}
\vspace{-0.15cm}

\paragraph{Query-key product in attention.}
Consider an attention layer with query head $h$ and its assigned key head $g$.
Let $W_{q,h},W_{k,g}\in\mathbb{R}^{r\times d}$ denote the query and key
projection matrices, where $r$ is the head dimension and $d$ is the hidden
dimension. Given hidden states $X\in\mathbb{R}^{n\times d}$,
\begin{align}
    Q_h &= XW_{q,h}^{\top},
    &
    K_g &= XW_{k,g}^{\top}.
    \label{eq:qk-projections}
\end{align}
The pre-softmax attention logits can then be written as
\begin{align}
    Q_hK_g^{\top}
    &=
    XW_{q,h}^{\top}W_{k,g}X^{\top}
     =
    XM_hX^{\top},
    \label{eq:qk-logits}
\end{align}
where the effective query-key product is
\begin{align}
    M_h
    =
    W_{q,h}^{\top}W_{k,g}
    \in \mathbb{R}^{d\times d}.
    \label{eq:qk-product}
\end{align}
Thus, $M_h$ is the operator directly controlling the attention interaction.
Since it is formed from two rank-$r$ factors, $\operatorname{rank}(M_h)\le r$,
so each head admits a low-rank spectral description.

\paragraph{Symmetric and antisymmetric decomposition.}
Any square QK product $M_h$ admits the unique decomposition
\begin{align}
    M_h &= S_h+\Omega_h,
    \label{eq:sym-skew-decomp}
    \\
    S_h &= \frac{1}{2}\left(M_h+M_h^{\top}\right),
    &
    \Omega_h &= \frac{1}{2}\left(M_h-M_h^{\top}\right),
    \label{eq:sym-skew-def}
\end{align}
where
\begin{align}
    S_h^{\top}=S_h,
    \qquad
    \Omega_h^{\top}=-\Omega_h.
    \label{eq:sym-skew-properties}
\end{align}
These two components are orthogonal under the Frobenius inner product:
\begin{align}
    \langle S_h,\Omega_h\rangle_F
    =
    \operatorname{tr}\!\left(S_h^{\top}\Omega_h\right)
    =
    0,
    \label{eq:sym-skew-orthogonal}
\end{align}
and therefore the QK-product energy decomposes additively:
\begin{align}
    \|M_h\|_F^2
    =
    \|S_h\|_F^2
    +
    \|\Omega_h\|_F^2.
    \label{eq:sym-skew-energy}
\end{align}

For token representations $x_i$ and $x_j$, the logit contribution separates as
\begin{align}
    x_i^{\top}M_hx_j
    =
    x_i^{\top}S_hx_j
    +
    x_i^{\top}\Omega_hx_j.
    \label{eq:logit-sym-skew}
\end{align}
The symmetric term is reciprocal,
\begin{align}
    x_i^{\top}S_hx_j
    =
    x_j^{\top}S_hx_i,
    \label{eq:symmetric-mutual}
\end{align}
whereas the antisymmetric term changes sign under direction reversal:
\begin{align}
    x_i^{\top}\Omega_hx_j
    =
    -
    x_j^{\top}\Omega_hx_i.
    \label{eq:skew-directional}
\end{align}
Thus, $S_h$ captures mutual content-similarity patterns, while $\Omega_h$
captures directional attention patterns.

\begin{figure*}
    \centering
    \includegraphics[width=1.0\linewidth]{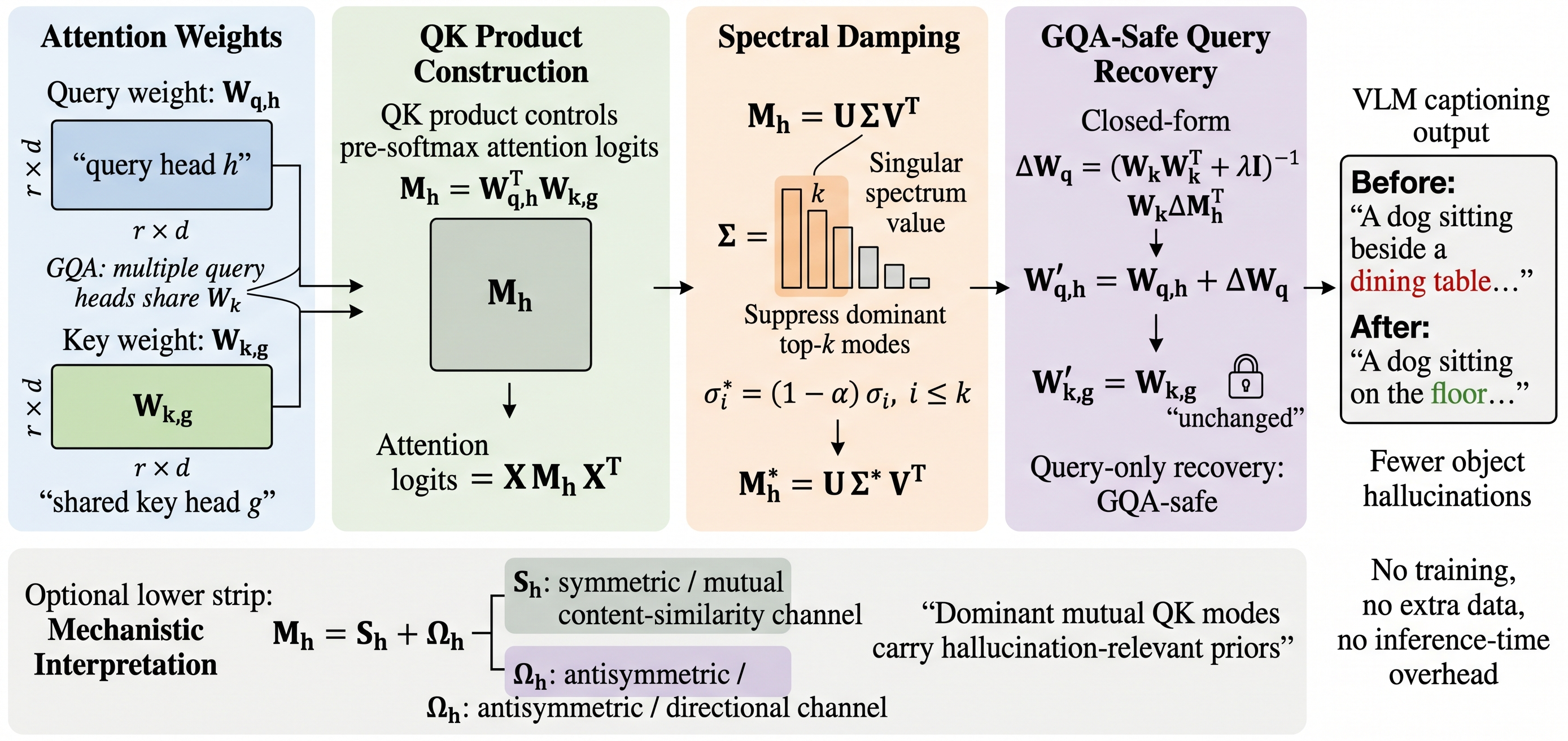}
    \vspace{-0.6cm}
    \caption{\textbf{Overview of QK Product Steering}. For each selected query head and its assigned shared key head, we form the QK product $M_h=W_{q,h}^{\top}W_{k,g}$, suppress its dominant singular modes, and recover an edited query weight through a closed-form query-only projection. The shared key weight is kept unchanged, making the edit compatible with grouped-query attention. The same framework also enables symmetric and antisymmetric QK decompositions to identify whether hallucination-relevant priors arise from mutual or directional attention channels.}
    \label{fig:main_figure}
    \vspace{-0.5cm}
\end{figure*}

\vspace{-0.2cm}
\section{Proposed Method}
\label{sec:method}
\vspace{-0.2cm}
We propose \emph{QK Product Steering}, a data-free and training-free weight edit for mitigating object hallucination in VLMs. The method is applied once to a pretrained model and then leaves inference unchanged, requiring no auxiliary model, fine-tuning, or decoding-time intervention.

\noindent \textbf{Motivation.} The core idea is to steer the query-key product, which directly controls the pre-softmax attention logits. We hypothesize that its dominant spectral directions encode strong vision-language priors: useful for fluent generation, but harmful when they override visual evidence and produce plausible yet absent objects. QK Product Steering therefore suppresses a small number of dominant modes in selected middle-layer attention heads.

Concretely, for each selected head, we compute the thin SVD of the query-key product, damp its top singular modes, and project the edited product back into the query weights through a closed-form query-only update. The key weights are kept fixed, making the edit compatible with grouped-query attention. This yields a localized spectral intervention that weakens prior-driven attention patterns while preserving most of the model's original computation; an overview is shown in Figure~\ref{fig:main_figure}.

\noindent \textbf{Thin SVD of the QK Product.}
Since $M_h \in \mathbb{R}^{d \times d}$ has rank at most $r$, we compute its thin SVD from the smaller factors. Using economy QR decompositions,
\begin{align}
    W_{q,h}^{\top} &= Q_q R_q,
    &
    W_{k,g}^{\top} &= Q_k R_k,
\end{align}
where $Q_q,Q_k \in \mathbb{R}^{d \times r}$ have orthonormal columns and $R_q,R_k \in \mathbb{R}^{r \times r}$. Then
\begin{align}
    M_h
    =
    W_{q,h}^{\top}W_{k,g}
    =
    Q_q C_h Q_k^{\top},
    \qquad
    C_h = R_qR_k^{\top}.
\end{align}
Taking the SVD $C_h=\widehat{U}\Sigma\widehat{V}^{\top}$ gives
\begin{align}
    M_h = U\Sigma V^{\top},
    \qquad
    U=Q_q\widehat{U},
    \qquad
    V=Q_k\widehat{V}.
\end{align}
Thus, we obtain the singular modes of $M_h$ from an $r \times r$ SVD instead of a $d \times d$ decomposition.

\noindent \textbf{Dominant Mode Suppression.} Let
\begin{align}
    M_h = U\Sigma V^{\top},
    \label{eq:mh-svd}
\end{align}
be the thin SVD of the QK product, where $\Sigma=\mathrm{diag}(\sigma_1,\ldots,\sigma_r)$ with $\sigma_1 \geq \cdots \geq \sigma_r$. We edit a mode set
$\mathcal{S}\subseteq\{1,\ldots,r\}$, with the default choice
$\mathcal{S}=\{1,\ldots,k\}$ corresponding to the top-$k$ modes. Given
damping strength $\alpha$, we define
\begin{align}
    \Sigma^{\star}_{ii}
    =
    \begin{cases}
    (1-\alpha)\Sigma_{ii}, & i \in \mathcal{S},\\
    \Sigma_{ii}, & i \notin \mathcal{S},
    \end{cases}
    \label{eq:sigma-damping}
\end{align}
and reconstruct the edited product as
\begin{align}
    M_h^{\star}=U\Sigma^{\star}V^{\top}.
    \label{eq:edited-product}
\end{align}
Here $\alpha=0$ gives no edit, while $\alpha=1$ fully removes the selected
modes. In our main setting, we suppress the top singular modes in middle
layers, targeting the strongest attention directions while preserving most
of the QK product structure. This produces an edited product $M_h^\star$ in the spectral space. To make the edit usable in the original VLM, we must realize $M_h^\star$ as an update to the model weights without disrupting shared key projections.

\noindent \textbf{GQA-Safe Query Recovery.} After editing $M_h$, we recover model weights while keeping the shared key
projection fixed. This avoids conflicts under grouped-query attention. Let
\begin{align}
    \Delta M_h = M_h^{\star}-M_h .
    \label{eq:delta-mh}
\end{align}
We compute the smallest query-weight update that realizes the edited product:
\begin{equation}
\begin{aligned}
    \min_{\Delta W_{q,h}}\quad
    & \|\Delta W_{q,h}\|_F^2 \\
    \text{s.t.}\quad
    & (W_{q,h}+\Delta W_{q,h})^{\top}W_{k,g}
    = M_h^{\star}.
\end{aligned}
\label{eq:query-projection-objective}
\end{equation}
The ridge-stabilized closed-form update is
\begin{align}
    \Delta W_{q,h}
    &=
    \left(
    W_{k,g}W_{k,g}^{\top}+\lambda I
    \right)^{-1}
    W_{k,g}\Delta M_h^{\top},
    \label{eq:query-update}
    \\
    \lambda
    &=
    \epsilon \cdot \frac{1}{r}
    \mathrm{tr}\!\left(W_{k,g}W_{k,g}^{\top}\right).
    \label{eq:ridge-lambda}
\end{align}
The final recovered weights are:
\begin{align}
    W_{q,h}^{\star}=W_{q,h}+\Delta W_{q,h},
    \qquad
    W_{k,g}^{\star}=W_{k,g}.
    \label{eq:final-qk-weights}
\end{align}
Thus, the edited QK product is implemented through a query-only update, while
the shared key weights remain unchanged.

\noindent \textbf{Symmetric and Antisymmetric Edit Variants.} Building on the decomposition in Section~\ref{sec:preliminaries}, we use the symmetric and antisymmetric components of $M_h$ as diagnostic edit targets. These variants provide a mechanistic view of which attention channel contributes to hallucination. Given $M_h = S_h + \Omega_h$, we define three edit variants as shown in Table~\ref{tab:sym-skew-variants}.

Each edited product $M_h^\star$ is mapped back to the query weights using the same GQA-safe query recovery step. These variants let us test whether hallucination is mainly associated with mutual content-similarity patterns in $S_h$ or directional attention patterns in $\Omega_h$ for mechanistic interpretability.

    

\begin{wraptable}{l}{0.52\columnwidth}
    \centering
    \vspace{-0.3cm}
    \scriptsize
    \setlength{\tabcolsep}{3pt}
    \renewcommand{\arraystretch}{1.08}
    \begin{tabularx}{0.50\columnwidth}{lXX}
    \toprule
    \textbf{Variant} & \textbf{Edit} & \textbf{Product} \\
    \midrule
    Sym-only 
    & Damp top-$k_s$ eigenmodes of $S_h$ 
    & $M_h^\star = S_h^\star + \Omega_h$ \\
    
    Antisym-only
    & Damp top-$k_a$ singular modes of $\Omega_h$
    & $M_h^\star = S_h + \Omega_h^\star$ \\
    
    Sym+Antisym
    & Edit both $S_h$ and $\Omega_h$
    & $M_h^\star = S_h^\star + \Omega_h^\star$ \\
    \bottomrule
    \end{tabularx}
    \vspace{-0.2cm}
    \caption{Symmetric and antisymmetric edit variants. Sym-only edits the mutual content-similarity channel $S_h$, Antisym-only edits the directional channel $\Omega_h$, and Sym+Antisym edits both.}
    \label{tab:sym-skew-variants}
    \vspace{-0.5cm}
\end{wraptable}

\noindent \textbf{Layer and Head Selection.} We apply QK Product Steering to selected layers and heads. Empirically, middle layers provide the best intervention point: early-layer edits can disrupt basic visual-textual representations, while late-layer edits have limited effect on object hallucination. Middle layers are therefore well positioned to weaken hallucination-related attention patterns while preserving caption quality. For each selected layer, we edit every query head using its assigned shared key head under grouped-query attention. The shared key weights are kept fixed throughout, and only the corresponding query weights are updated. A detailed discussion is provided in Section~\ref{sec:experiment_and_results}.

\noindent \textbf{Interpretation and Theoretical Insights.}
\emph{QK Product Steering} is motivated by the structure of attention itself. The QK product is the exact bilinear operator that produces pre-softmax attention logits, so editing it directly steers attention behavior. Its top singular modes correspond to the highest-energy low-rank directions of this operator, making their suppression a principled way to weaken dominant attention patterns rather than applying an arbitrary weight perturbation. In VLMs, these dominant directions can behave like generic vision-language priors: useful for fluent generation, but harmful when they override visual evidence and produce unsupported objects. By suppressing only a few dominant modes in middle layers, QK Product Steering weakens such prior-driven attention while preserving most of the model's computation.

The symmetric--antisymmetric decomposition further separates mutual content-similarity patterns from directional attention patterns providing more interpretability. This provides a mechanistic way to test where hallucination-relevant priors reside. Finally, because the edit is low-rank, localized to selected middle layers, and recovered through a minimum-change query-only update, its off-target effect is bounded (more details in Appendix ~\ref{app:theory}).

\begin{figure*}[t]
  \centering
  \includegraphics[width=\textwidth]{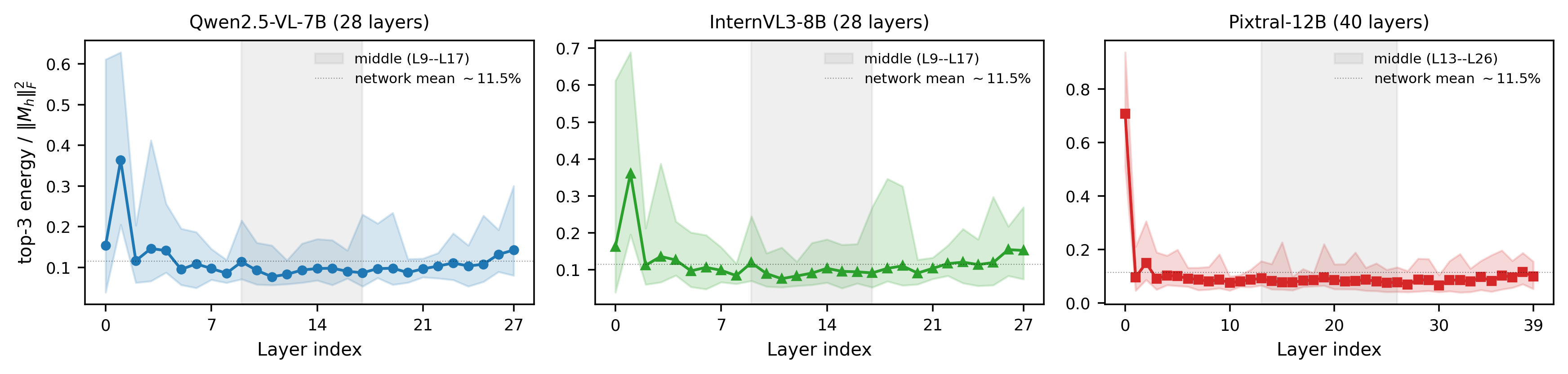}
  \vspace{-0.6cm}
  \caption{Per layer and head, the fraction of $\Mh$'s Frobenius energy in its top-$3$ singular modes, $E_3 = (\sigma_1^2 + \sigma_2^2 + \sigma_3^2)/\|\Mh\|_F^2$. Solid line: mean over query heads in each layer; shaded band: head min/max; grey vertical band: middle-third edit locus. A small $E_3$ means QK Product Steering top-$3$ touches a small share of the operator (a surgical edit); a large $E_3$ means it removes a large share (a destructive edit). Across all three models, $E_3$ is highest in the first few layers and lowest in the middle third, where the same low-rank intervention is most surgical.}
  \label{fig:top3_energy_per_layer}
  \vspace{-0.5cm}
\end{figure*}

\vspace{-0.2cm}
\section{Experiments and Results}
\label{sec:experiment_and_results}
\vspace{-0.2cm}

\textbf{Experimental Setup.} We evaluate QK Product Steering on three grouped-query-attention vision-language models: Qwen2.5-VL-7B~\cite{qwen25vl}, InternVL3-8B~\cite{internvl3}, and Pixtral-12B~\cite{pixtral2024}. Qwen2.5-VL-7B is the primary model where the full hyperparameter sweep is run; the other two are used to test cross-architecture generalization at each model's preferred $(k, \alpha)$. Per-model architecture details (layer counts, head dimensions, GQA group sizes) and the hardware configuration used to run all experiments are deferred to Appendix~\ref{sec:appendix-models-hardware}.

Our evaluation is generative object hallucination, measured with CHAIR~\cite{rohrbach2018chair} on COCO 2014 \texttt{val}; all sweeps and ablations in the paper are scored on CHAIR. POPE~\cite{li2023pope} and HallusionBench~\cite{guan2023hallusion} are reported alongside as probe-based hallucination benchmarks (templated yes/no and adversarial paired questions, respectively), and MME~\cite{fu2023mme} and MMMU val~\cite{yue2024mmmu} are reported only to verify that general capability is preserved. Our probe and capability benchmarks are run through \texttt{lmms-eval}~v0.7~\cite{zhang-etal-2025-lmms} with each model's reference inference config. Per-benchmark sample sizes, decoding settings, and scoring protocols are deferred to Appendix~\ref{sec:appendix-additional}.

Decoding is greedy throughout; hallucination deltas are reported with paired-bootstrap CIs (iteration count per experiment in Appendix~\ref{sec:statistical-significance}) and capability deltas as raw within-pipeline differences. Controls and protocol details are in Appendix ~\ref{sec:appendix-additional}.

\noindent\textbf{Why edit the middle layers? Top-$k$ surgicality across depth.}\label{sec:spectral_diag}
QK Product Steering damps the top-$k$ singular modes of $\Mh = \Wq^\top \Wk$, and the natural a priori measure of how \emph{surgical} that edit is at a given layer is the fraction of the operator's Frobenius energy carried by those modes, $E_k(\Mh) = \sum_{i\le k} \sigma_i^2 / \|\Mh\|_F^2$. A head with $E_3 = 0.10$ has $90\%$ of its operator left untouched by QK Product Steering top-$3$; a head with $E_3 = 0.36$ loses a third of the operator. We compute $E_3$ per head for every query head across all three models and use it to localise the edit before running any CHAIR experiment.

Figure~\ref{fig:top3_energy_per_layer} reveals a single, model-consistent depth pattern: $E_3$ rises sharply over the first few layers, falls to a model-internal minimum in the middle third, and creeps back up toward the output. Early layers are rank-concentrated, with $E_3$ reaching as much as $0.36$ on Qwen and $0.71$ on Pixtral, so damping their top three modes removes a substantial fraction of the operator and collapses generation. The middle band is the unique cell where the same low-rank intervention is small relative to $\|\Mh\|_F^2$ and therefore structurally consistent with a surgical edit. We adopt the middle-third locus as the default edit band on this structural basis alone, before any CHAIR experiment.

\noindent\textbf{Main Results.} Table~\ref{tab:main} reports QK Product Steering and its channel decompositions across three GQA vision-language models and six benchmarks, each at the model's preferred $(k, \alpha)$ in the middle-third edit band. The ablations in Section~\ref{sec:ablations} drive the $(k{=}3, \alpha{=}1)$ default on Qwen; for cross-architecture transfer we keep $\alpha{=}1$ and the middle-third band fixed and select $k$ per model by a small CHAIR sweep, with details in Appendix~\ref{sec:appendix-cross-arch-k}. QK Product Steering top-$k$ is the lowest-CHAIR$_s$ non-control row on every model, while the matched random-$k$ control at the same $k$ moves CHAIR$_s$ negligibly on all three -- ruling out arbitrary low-rank perturbations of $\Wq$ and identifying the dominant query-key modes as the carrier of the hallucination signal. Paired-bootstrap CIs supporting this picture, including the symmetric-channel edit on Qwen and the QK Product Steering top-$k$ edit on Pixtral, are reported in Appendix~\ref{sec:statistical-significance}. The channel decompositions separate cleanly: Sym-only TOP captures most of the CHAIR$_s$ gain on Qwen and InternVL3, whereas Antisym-only TOP is the worst non-control row on CHAIR$_s$ across all three models -- antisymmetric content carries no hallucination signal under this edit. Capability is broadly preserved within the random-control noise budget; the only systematic exception is the Pixtral middle band at $k{=}3$, where MMMU is sensitive to which spectral channel is touched (Sym best, Antisym worst). Per-method $(k, \alpha)$ sweeps and bottom-$k$ / amplification controls are reported in Appendix~\ref{sec:appendix-additional}.

\begin{table*}[t]
  \centering\small
  \setlength{\tabcolsep}{4pt}
  \resizebox{\textwidth}{!}{%
  \begin{tabular}{llcccccc}
    \toprule
       &      & \multicolumn{2}{c}{Hallucination ($\downarrow$)} & \multicolumn{2}{c}{Probe ($\uparrow$)} & \multicolumn{2}{c}{Capability ($\uparrow$)} \\
    \cmidrule(lr){3-4}\cmidrule(lr){5-6}\cmidrule(lr){7-8}
    Model & Method & CHAIR$_s$ & CHAIR$_i$ & POPE F1 & HB \textsc{aAcc} & MME total & MMMU \\
    \midrule
    \multirow{5}{*}{Qwen2.5-VL-7B}
      & Baseline                        & 31.44             & 65.26             & 86.18 & 64.30 & 2339.2             & \textbf{51.56} \\
      & QK Product Steering random $k{=}3$             & 31.33             & 64.79             & 85.33          & 64.30          & 2332.8             & 51.33 \\
      & QK Product Steering top $k{=}3$                & \textbf{28.56}    & \textbf{63.75}    & \textbf{86.68}          & \textbf{64.98}          & 2293.3             & 50.78 \\
      & Sym-only TOP $k_s{=}3$          & 29.53   & 65.75   & 86.50          & 63.68          & \textbf{2343.2}    & 50.89 \\
      & Antisym-only TOP $k_a{=}6$         & 31.60   & 66.73   & 82.87          & 61.12          & 2290.9             & 49.44 \\
    \midrule
    \multirow{5}{*}{InternVL3-8B}
      & Baseline                        & 33.81             & 62.98             & 90.85          & 60.58 & \textbf{2388.0}    & 54.33 \\
      & QK Product Steering random $k{=}1$             & 33.97             & 63.02             & 90.95          & 60.32          & 2367.0             & 54.78 \\
      & QK Product Steering top $k{=}1$                & \textbf{33.23}    & \textbf{62.98}    & 91.09          & 60.14          & 2377.2             & 54.67 \\
      & Sym-only TOP $k_s{=}1$          & 33.43             & 63.20             & \textbf{91.14} & \textbf{60.67}          & 2381.4    & \textbf{55.78} \\
      & Antisym-only TOP $k_a{=}2$         & 33.98             & 63.79             & 90.88          & 59.79          & 2371.7             & 54.89 \\
    \midrule
    \multirow{5}{*}{Pixtral-12B}
      & Baseline                        & 34.14             & 64.86             & 84.99          & 72.26          & 2009.1             & \textbf{45.44} \\
      & QK Product Steering random $k{=}3$             & 34.00             & 64.51             & 84.44          & 71.59          & 1963.2             & 45.67 \\
      & QK Product Steering top $k{=}3$                & \textbf{33.81}    & \textbf{63.04}    & \textbf{85.23} & \textbf{72.93} & \textbf{2018.6}    & 43.11 \\
      & Sym-only TOP $k_s{=}3$          & 34.10             & 64.70             & 85.21          & 72.71          & 1975.2             & 44.11 \\
      & Antisym-only TOP $k_a{=}6$         & 35.09             & 64.50             & 84.24          & 71.38          & 1995.4             & 42.46 \\
    \bottomrule
  \end{tabular}}
  \vspace{-0.2cm}
  \caption{\textbf{Main Results.} Results across three VLMs and six benchmarks. Paired-bootstrap CIs and statistical significance for CHAIR rows are reported in Appendix~\ref{sec:statistical-significance}; full benchmark configurations and evaluation details are provided in Appendix~\ref{sec:appendix-additional}.}
  \label{tab:main}
  \vspace{-0.5cm}
\end{table*}

\vspace{-0.2cm}
\section{Ablations}
\label{sec:ablations}
\vspace{-0.2cm}

The spectral diagnostic of Section~\ref{sec:spectral_diag} identifies \emph{which layers} carry the profile compatible with a small-$k$ surgical edit; it does not fix the mode count $k$, the damping strength $\alpha$, or empirically validate the layer choice. We close those gaps with three ablations on Qwen2.5-VL-7B: a layer-locus check, a $k$-sweep, and an $\alpha$-sweep. Together they fix the default configuration $(k{=}3,\alpha{=}1,\text{middle})$ used in our main results (see Table~\ref{tab:main}).

\noindent\textbf{Layer Locus.}\label{sec:layer_select}
Table~\ref{tab:layers_v2} matches the $E_3$ prediction of Figure~\ref{fig:top3_energy_per_layer}: early-band and full-stack edits collapse generation entirely (zero detectable COCO objects per caption); late-band edits preserve fluency but do not move CHAIR$_s$; only the middle band converts the surgical operator footprint into a measurable hallucination reduction.

\noindent\textbf{Mode Count.}\label{sec:ablations_k}
At fixed $\alpha{=}1$ on Qwen, the symmetric-channel top-$k$ edit peaks at $k_s{=}3$ on CHAIR$_s$ and saturates beyond it, while the antisymmetric channel is non-monotone and weak across all $k_a$ tested (Appendix Table~\ref{tab:appendix-k-sweep}). Matched-rank \emph{bottom}-$k$ controls are null on both channels (CIs include zero), confirming that the signal lives in the dominant modes of $S_h$ rather than in arbitrary low-rank perturbations or in the antisymmetric channel. On the full joint $(k, \alpha)$ grid (Appendix Section ~\ref{sec:appendix-joint-grid}), CHAIR$_s$ decreases monotonically with $k$ at $\alpha \ge 1$, but at increasing capability cost which we examine next.

\begin{wraptable}{l}{0.55\columnwidth}
  \centering
  \vspace{-0.3cm}
  \small
  \setlength{\tabcolsep}{4pt}
  \begin{tabular}{lccc}
    \toprule
    Layers          & $\Delta$CHAIR$_s$ & $\Delta$CHAIR$_i$ & avg\_obj \\
    \midrule
    early           & $-28.81$          & $-66.60$          & $0.00$ {\scriptsize (collapse)} \\
    \textbf{middle} & $\mathbf{-3.04}$  & $-2.06$           & $2.67$ \\
    late            & $+0.09$           & $-2.41$           & $2.62$ \\
    all             & $-28.81$          & $-66.60$          & $0.00$ {\scriptsize (collapse)} \\
    \bottomrule
  \end{tabular}
  \vspace{-0.2cm}
  \caption{Layer sensitivity on Qwen2.5-VL-7B at $k{=}3, \alpha{=}1$, $N{=}2000$. \texttt{avg\_obj}: mean COCO objects per caption; baseline $2.54$.}
  \label{tab:layers_v2}
  \vspace{-0.4cm}
\end{wraptable}

\noindent\textbf{Damping Strength and Capability Tradeoff.}\label{sec:ablations_alpha}
At fixed $k{=}3$ on Qwen, the Sym-only edit's CHAIR$_s$ response is monotone in $\alpha$, with CIs excluding zero from $\alpha{=}0.5$ onward (Appendix Table~\ref{tab:appendix-alpha-sweep}); Antisym-only on $\Omega$ is flat across the sweep. Pushing past zero-out into $\alpha > 1$ continues to lower CHAIR$_s$ on the joint grid, but at a measurable capability cost on MMMU and MME; the same cost appears, at smaller magnitude, when $k$ is increased at safe $\alpha{=}1$ (Appendix Table~\ref{tab:appendix-capability-tradeoff}). $\alpha$ is the larger lever on capability. We therefore report $(k{=}3,\alpha{=}1)$ as the default: the smallest principled cell where the CHAIR drop is CI-significant and capability is preserved within the random-control noise budget.

\vspace{-0.2cm}
\section{Discussion and Analysis}
\label{sec:mechanism}
\vspace{-0.2cm}

QK Product Steering acts primarily on the symmetric part of the per-head QK product. On Qwen and InternVL3, damping the dominant eigenmodes of $S_h$ alone reproduces roughly two thirds of the CHAIR$_s$ reduction obtained by damping the full product; on Pixtral the symmetric contribution is smaller. Editing only the antisymmetric part $\Omega_h$ does \emph{not} reduce hallucination and, across all three models, degrades probe and capability metrics relative to baseline (Table~\ref{tab:main}). Antisymmetric content therefore appears to carry no hallucination signal but still couples to general task performance, so editing it costs capability without any hallucination return. The behavioural consequence of that damping is non-local: at the edited layers themselves the layer-mean attention shift is essentially zero -- per-head reorganisations cancel within the layer -- and the change surfaces only in the unedited layers downstream of the edit site, as a small, uniform tilt of attention probability toward image tokens (Section~\ref{sec:trajectory}). That per-step tilt is amplified across the roughly two-hundred decoding steps of a free-form caption into a CHAIR-level effect, but is too small to flip a single short answer on MME or MMMU, which is why hallucination drops while capability is preserved (Section~\ref{sec:freq}).

Two natural null hypotheses fail under teacher forcing: the edit neither selectively suppresses hallucinated-noun log-probabilities at fixed prefix, nor re-grounds attention at the pre-noun position (Appendix~\ref{sec:appendix-nulls}). The mechanism is neither token- nor position-local.

\noindent\textbf{Editing the middle layers shifts attention to image tokens in the \emph{later} layers.}\label{sec:trajectory}
QK Product Steering's effect on attention has two structural features that make the discussion follow: it is \emph{small} per layer and \emph{non-local} in depth. The layer-mean image-attention shift at the edited layers is essentially zero; the signal appears as a uniform tilt toward image tokens in every unedited downstream layer, around $0.4$pp per layer. Both features are essential: the locus tells us \emph{where} the mechanism lives (downstream of the edit, not at the edit), and the magnitude tells us \emph{which benchmarks} can resolve it (long-horizon captioning yes, single-decision MCQ no). We measure the shift under teacher forcing -- identical $[\text{prompt} + \text{baseline caption}]$ sequence through both edited and unedited models, so any attention difference is attributable to the weight edit alone.

Upstream of the edit ($L_0$--$L_8$) the shift is bit-exactly zero -- a protocol check, since unedited layers with identical input must produce identical attention. Inside the edited band ($L_9$--$L_{17}$) QK Product Steering's layer-mean is effectively zero ($+0.05$pp, sign-mixed). The signal lives in the downstream unedited layers ($L_{18}$--$L_{27}$): every one of them shifts in the same direction, with mean $+0.4$pp. Figure~\ref{fig:downstream_dv} resolves this layer-by-layer: the largest shifts are at $L_{19}$--$L_{20}$ immediately after the edit, taper through $L_{22}$, and hold a smaller positive offset to $L_{27}$ -- the decay shape of a perturbation that enters at $L_{17}$ and propagates through subsequent attention and MLP blocks. Sym-only shows the same downstream pattern at comparable magnitude; Antisym-only, which is inert on CHAIR, has no such downstream signal (Appendix~\ref{sec:appendix-attn-trajectory}, Figure~\ref{fig:attention_delta}).


The zero layer-mean inside the edited band reflects per-head cancellation rather than absence of effect: per-head shifts reach $\pm 0.48$pp with $138$ positive and $114$ negative cells across the $252$ (layer, head) cells in the edited band, summing to a near-zero layer mean but a non-zero residual stream that the downstream layers respond to (see Appendix Section ~\ref{sec:appendix-attn-trajectory}, Figure~\ref{fig:per_head_dv}).

The per-layer magnitude of $\Delta v$ is small by design of the mechanism, not by limitation of the edit. A shift of this size cannot flip a single short answer on MME or MMMU; applied at every one of the $\sim 200$ greedy decoding steps of a free-form caption, it accumulates into the CHAIR-level effect we observe. The next subsection makes that asymmetry concrete.

\begin{wrapfigure}{r}{0.52\columnwidth}
  \centering
  \vspace{-0.3cm}

  \begin{subfigure}[b]{0.5\columnwidth}
    \centering
    \includegraphics[width=\linewidth]{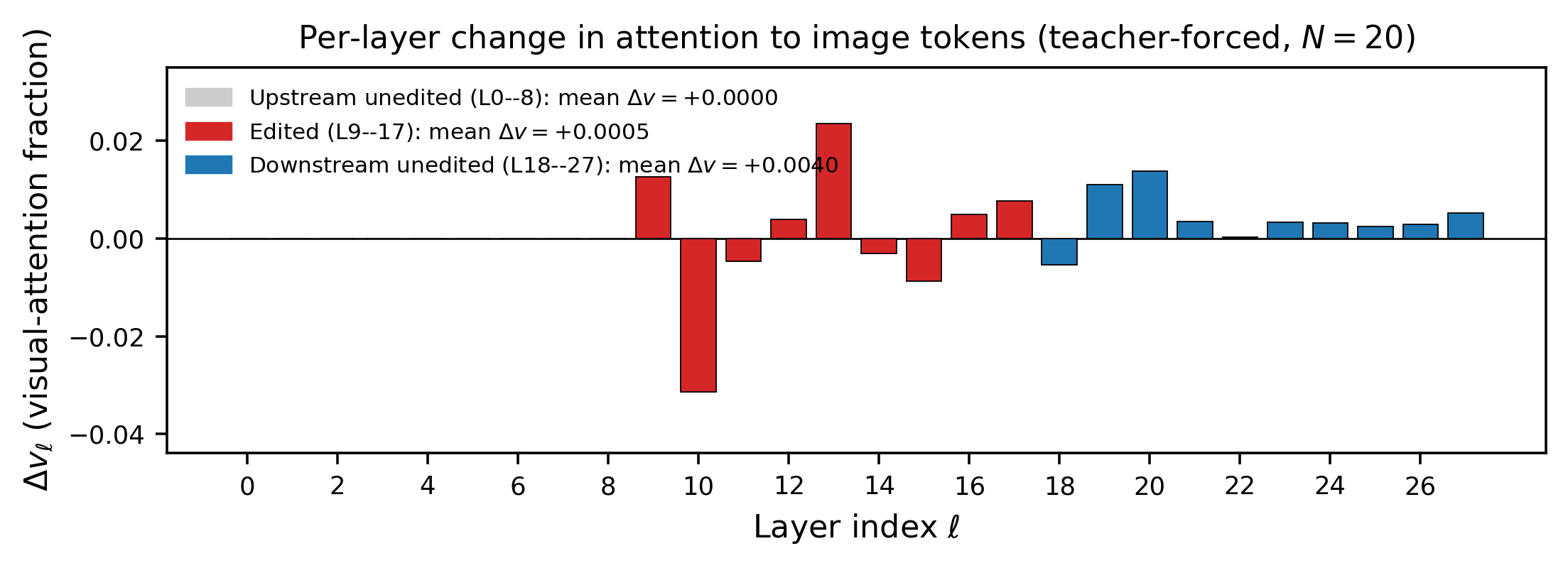}
    \caption{QK Product Steering top-$3$, $N{=}200$.}
    \label{fig:attention_delta_body}
  \end{subfigure}

  \vspace{0.1cm}

  \begin{subfigure}[b]{0.5\columnwidth}
    \centering
    \includegraphics[width=\linewidth]{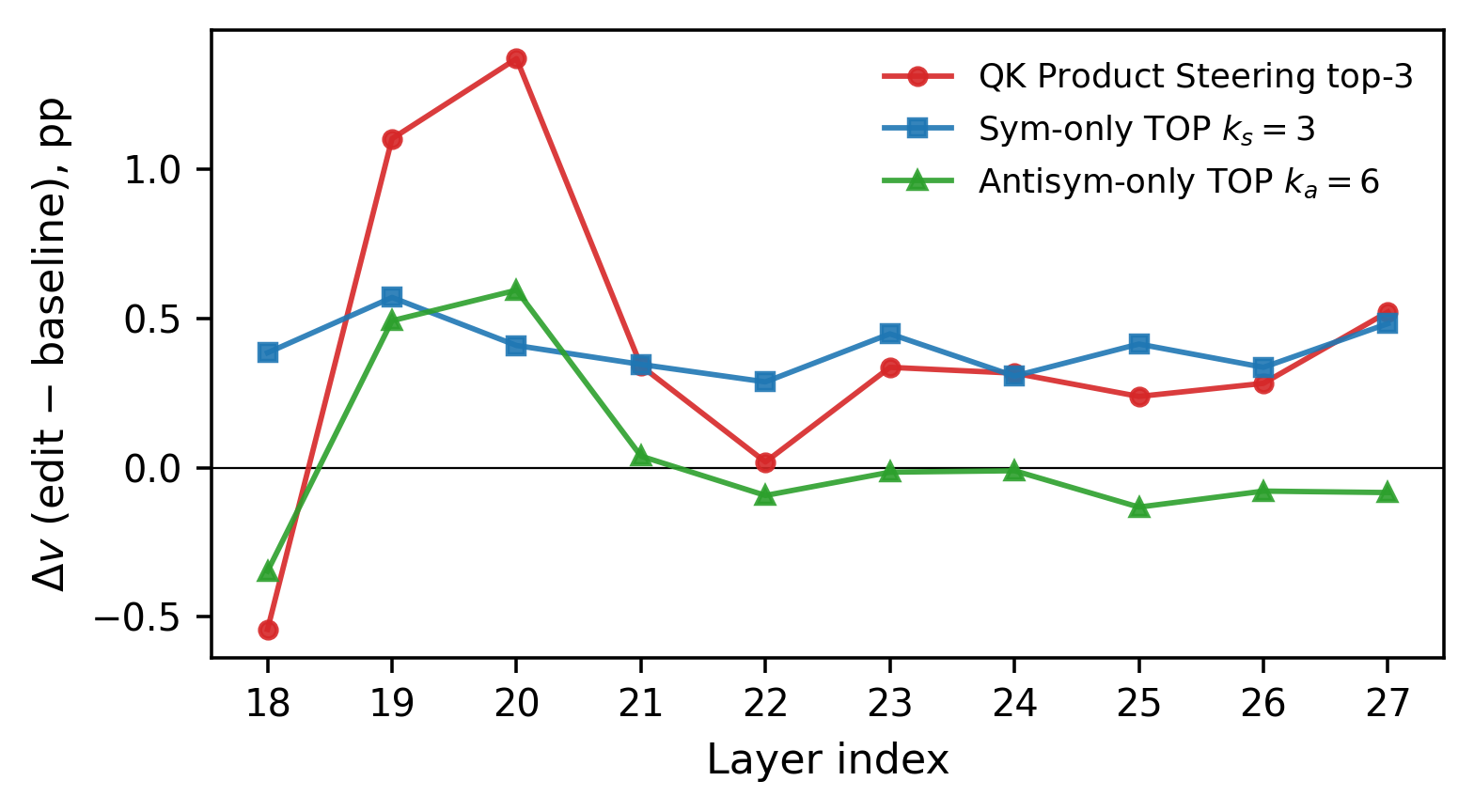}
    \caption{Three spectral methods, $N{=}200$.}
    \label{fig:downstream_dv}
  \end{subfigure}

  \vspace{-0.2cm}
  \caption{Per-layer attention-to-image shift $\Delta v$ (edit $-$ baseline, percentage points) under teacher forcing. (\textbf{Top}) (a) Full 28-layer profile for QK Product Steering: upstream is bit-exactly zero, edited band cancels in the layer mean, downstream is uniformly positive. (\textbf{Bottom}) (b) Downstream zoom across the three spectral methods: QK Product Steering and Sym-only both lift attention to image tokens (mean $+0.4$pp); Antisym-only is flat.}
  \label{fig:attention_combined}
  \vspace{-0.5cm}
\end{wrapfigure}

\noindent\textbf{Why the attention shift reduces CHAIR but not MME or MMMU.}\label{sec:freq}
The per-layer image-attention shift of Section~\ref{sec:trajectory} is small and applies to every forward pass equally, yet it produces a clean CHAIR$_s$ reduction on free-form captioning while leaving MME and MMMU val essentially flat. Two factors account for the asymmetry. The first is decoding length: CHAIR scores captions of $\sim 200$ greedy steps, so a small per-step bias toward image-grounded continuations accumulates over the trajectory and concentrates at the positions where the baseline model was borderline between a hallucinated and a grounded noun. MME and MMMU score a single short answer dominated by the recognition head over a handful of candidate tokens; there is no long generation to integrate over. The second is which spectral subspace QK Product Steering damps: the top modes of $\Mh$ are the high-energy, generic content-similarity defaults the model leans on when image evidence is weak -- the same defaults that drive hallucination on free-form captioning. Damping them forces decoding toward image-specific signal without touching the lower-energy modes that the recognition-style benchmarks reward, so the cost surfaces as a free-form hallucination drop with capability preserved.

\noindent\textbf{Qualitative Examples.} Appendix Section ~\ref{sec:appendix-qual} provides qualitative COCO examples showing baseline and edited captions side by side for comparison. Across the full validation set, QK Product Steering top-$3$ meets a strict improvement criterion on $756$ images: it removes at least one hallucinated object, introduces no new hallucinations, and preserves all grounded object mentions.

\vspace{-0.3cm}
\section{Conclusion}
\label{sec:conclusion}
\vspace{-0.3cm}
We introduced \emph{QK Product Steering}, a training-free and zero-inference-cost method for reducing object hallucination in VLMs through query-key product editing. By suppressing a few dominant middle-layer modes, it weakens prior-driven attention while keeping architecture and inference unchanged. Our results show that these modes are closely linked to hallucination, making QK Product Steering a practical alternative to decoding-time interventions.

\vspace{-0.3cm}
\section{Limitations and Future Work}
\label{sec:limitations}
\vspace{-0.3cm}
\emph{QK Product Steering} is a lightweight post-hoc edit with a focused scope. It targets object hallucination in free-form VLM generation by weakening dominant query-key attention modes, but it does not add new visual knowledge or improve the model's underlying perception. Therefore, failures caused by poor object recognition, ambiguous images, or missing visual evidence may remain.

The method also uses a small set of hyperparameters, including the edited layers, number of modes, and damping strength. Although we find that middle-layer dominant modes work well across multiple VLMs, the optimal configuration may vary across architectures. Finally, our edit operates on the content-based QK product and does not explicitly model position-dependent effects such as rotary positional embeddings. Extending the method to adaptive and position-aware steering is an important direction for future work.

\bibliography{colm2026_conference}
\bibliographystyle{colm2026_conference}

\etocdepthtag.toc{mtmain}

\appendix
\etocdepthtag.toc{mtappendix}

\etocsettagdepth{mtmain}{none}
\etocsettagdepth{mtappendix}{subsection}

\etocsettocstyle{\section*{Appendix Contents}\vspace{0.5em}}{}

\onecolumn
\newpage
\tableofcontents
\newpage

\section{Notation and Conventions}
\label{app:notation}

Table~\ref{tab:notation} summarizes the main notation used throughout the paper. We use $h$ to index query heads and $g$ to index the corresponding shared key head under grouped-query attention. The central object in our method is the per-head QK product $M_h=W_{q,h}^{\top}W_{k,g}$, which directly determines the pre-softmax attention logits. QK Product Steering edits this product spectrally, maps the edited product back to the query weights, and keeps the shared key weights fixed.

\begin{table}[h]
    \centering
    \small
    \setlength{\tabcolsep}{4pt}
    \renewcommand{\arraystretch}{1.08}
    \begin{tabular}{ll}
    \toprule
    \textbf{Notation} & \textbf{Description} \\
    \midrule
    $X \in \mathbb{R}^{n \times d}$ 
    & Hidden states with sequence length $n$ and dimension $d$ \\
    
    $r$ 
    & Attention head dimension \\
    
    $d$ 
    & Model hidden dimension \\
    
    $h$ 
    & Query head index \\
    
    $g$ 
    & Shared key head index assigned to query head $h$ \\
    
    $W_{q,h} \in \mathbb{R}^{r \times d}$ 
    & Query projection for head $h$ \\
    
    $W_{k,g} \in \mathbb{R}^{r \times d}$ 
    & Key projection for shared key head $g$ \\
    
    $Q_h = XW_{q,h}^{\top}$ 
    & Query representation for head $h$ \\
    
    $K_g = XW_{k,g}^{\top}$ 
    & Key representation for head $g$ \\
    
    $M_h = W_{q,h}^{\top}W_{k,g}$ 
    & Per-head QK product controlling attention logits \\
    
    $S_h$ 
    & Symmetric component of $M_h$; mutual attention channel \\
    
    $\Omega_h$ 
    & Antisymmetric component of $M_h$; directional attention channel \\
    
    $U\Sigma V^{\top}$ 
    & Thin SVD of the QK product $M_h$ \\
    
    $\sigma_i$ 
    & $i$-th singular value of $M_h$ \\
    
    $\mathcal{S}$ 
    & Set of selected modes to edit \\
    
    $k$ 
    & Number of singular modes selected for suppression \\
    
    $\alpha$ 
    & Damping strength for selected modes \\
    
    $M_h^{\star}$ 
    & Edited QK product after spectral damping \\
    
    $\Delta M_h = M_h^{\star}-M_h$ 
    & Desired change in the QK product \\
    
    $\Delta W_{q,h}$ 
    & Query-weight update used to realize $M_h^{\star}$ \\
    
    $\lambda$ 
    & Ridge coefficient for stable query recovery \\
    
    $W_{q,h}^{\star}$ 
    & Edited query projection \\
    
    $W_{k,g}^{\star}$ 
    & Key projection after editing; kept unchanged \\
    \bottomrule
    \end{tabular}
    \caption{Summary of notation used in QK Product Steering.}
    \label{tab:notation}
\end{table}

\vspace{-0.2cm}
\section{Theoretical Insights}
\label{app:theory}
\vspace{-0.2cm}

This section provides theoretical support for QK Product Steering. Our goal is not to prove universal hallucination removal, but to justify why editing the query-key product is a direct, targeted, and localized intervention. We show that the QK product is the exact attention-logit operator, that top singular modes are principled edit targets, that symmetric and antisymmetric components form orthogonal attention channels, and that the query-only recovery is a minimum-change GQA-safe realization of the edited product.

\subsection{QK Product as the Attention-Logit Operator}
\label{app:qk-operator}

For query head $h$ and its assigned key head $g$, let
\begin{align}
    Q_h = XW_{q,h}^{\top},
    \qquad
    K_g = XW_{k,g}^{\top},
\end{align}
where $X \in \mathbb{R}^{n \times d}$ is the hidden-state matrix. The pre-softmax attention logits are
\begin{align}
    Q_hK_g^{\top}
    =
    XW_{q,h}^{\top}W_{k,g}X^{\top}
    =
    XM_hX^{\top},
    \qquad
    M_h = W_{q,h}^{\top}W_{k,g}.
    \label{eq:app-qk-logit-operator}
\end{align}
Thus, $M_h$ is the content-based bilinear operator controlling the query-key interaction for head $h$; standard scaling and position-dependent transformations such as RoPE act around this product. Editing $M_h$ therefore directly modifies the attention interaction, rather than intervening only at the decoding logits or output probabilities.

\subsection{Why Top Singular Modes Are Natural Edit Targets}
\label{app:top-modes}

Let the thin SVD of the QK product be
\begin{align}
    M_h = \sum_{i=1}^{r}\sigma_i u_i v_i^{\top},
    \qquad
    \sigma_1 \geq \sigma_2 \geq \cdots \geq \sigma_r .
\end{align}
The top-$k$ component is
\begin{align}
    M_{h,k}
    =
    \sum_{i=1}^{k}\sigma_i u_i v_i^{\top}.
\end{align}

\begin{proposition}[Dominant QK modes are maximum-energy directions]
The matrix $M_{h,k}$ is the best rank-$k$ approximation of $M_h$ in Frobenius norm:
\begin{align}
    M_{h,k}
    =
    \arg\min_{\mathrm{rank}(A)\leq k}
    \|M_h-A\|_F .
\end{align}
Moreover,
\begin{align}
    \|M_{h,k}\|_F^2
    =
    \sum_{i=1}^{k}\sigma_i^2 .
\end{align}
\end{proposition}

\begin{proof}
This follows directly from the Eckart--Young--Mirsky theorem. The truncated SVD gives the optimal rank-$k$ approximation in Frobenius norm, and the Frobenius energy of an SVD component is the sum of squared singular values. Since the singular values are sorted in decreasing order, the top-$k$ modes carry the largest rank-$k$ spectral energy.
\end{proof}

QK Product Steering suppresses this dominant component:
\begin{align}
    M_h^{\star}
    =
    M_h-\alpha M_{h,k}
    =
    M_h-\alpha
    \sum_{i=1}^{k}\sigma_i u_i v_i^{\top}.
    \label{eq:app-topk-suppression}
\end{align}
Hence, the edit is not an arbitrary perturbation: it targets the strongest low-rank directions of the actual attention-logit operator.

\subsection{Symmetric and Antisymmetric Attention Channels}
\label{app:sym-skew-theory}

Every QK product admits the unique decomposition
\begin{align}
    M_h = S_h+\Omega_h,
    \qquad
    S_h=\frac{1}{2}(M_h+M_h^{\top}),
    \qquad
    \Omega_h=\frac{1}{2}(M_h-M_h^{\top}),
\end{align}
where $S_h^{\top}=S_h$ and $\Omega_h^{\top}=-\Omega_h$.

\begin{proposition}[Orthogonal mutual and directional channels]
The symmetric and antisymmetric components are orthogonal under the Frobenius inner product:
\begin{align}
    \langle S_h,\Omega_h\rangle_F
    =
    \mathrm{tr}(S_h^{\top}\Omega_h)
    =
    0.
\end{align}
Consequently,
\begin{align}
    \|M_h\|_F^2
    =
    \|S_h\|_F^2+\|\Omega_h\|_F^2.
\end{align}
\end{proposition}

\begin{proof}
Since $S_h^{\top}=S_h$ and $\Omega_h^{\top}=-\Omega_h$,
\begin{align}
    \mathrm{tr}(S_h^{\top}\Omega_h)
    =
    \mathrm{tr}(S_h\Omega_h).
\end{align}
Using trace invariance under transpose,
\begin{align}
    \mathrm{tr}(S_h\Omega_h)
    =
    \mathrm{tr}((S_h\Omega_h)^{\top})
    =
    \mathrm{tr}(\Omega_h^{\top}S_h^{\top})
    =
    \mathrm{tr}(-\Omega_h S_h)
    =
    -\mathrm{tr}(S_h\Omega_h).
\end{align}
Therefore, $\mathrm{tr}(S_h\Omega_h)=0$. The additive energy decomposition follows from orthogonality:
\begin{align}
    \|M_h\|_F^2
    =
    \|S_h+\Omega_h\|_F^2
    =
    \|S_h\|_F^2+\|\Omega_h\|_F^2
    +
    2\langle S_h,\Omega_h\rangle_F .
\end{align}
\end{proof}

For token representations $x_i$ and $x_j$, the attention-logit contribution decomposes as
\begin{align}
    x_i^{\top}M_hx_j
    =
    x_i^{\top}S_hx_j
    +
    x_i^{\top}\Omega_hx_j.
\end{align}
The symmetric term is reciprocal:
\begin{align}
    x_i^{\top}S_hx_j
    =
    x_j^{\top}S_hx_i,
\end{align}
whereas the antisymmetric term changes sign under reversal:
\begin{align}
    x_i^{\top}\Omega_hx_j
    =
    -
    x_j^{\top}\Omega_hx_i.
\end{align}
Thus, $S_h$ captures mutual content-similarity patterns, while $\Omega_h$ captures directional attention patterns. This justifies Sym-only and Antisym-only edits as mechanistic probes of whether hallucination is associated with reciprocal content priors or directional attention effects.

\subsection{GQA-Safe Query Recovery as a Minimum-Change Edit}
\label{app:query-recovery}

After obtaining an edited product $M_h^{\star}$, we must realize it in the model weights. Under grouped-query attention, the key projection $W_{k,g}$ is shared across multiple query heads, so modifying it can unintentionally affect other heads. We therefore keep $W_{k,g}$ fixed and update only $W_{q,h}$.

Let
\begin{align}
    \Delta M_h = M_h^{\star}-M_h .
\end{align}
We seek the smallest query-weight update that realizes the edited product:
\begin{equation}
\begin{aligned}
    \min_{\Delta W_{q,h}}\quad
    &\|\Delta W_{q,h}\|_F^2\\
    \text{s.t.}\quad
    &(W_{q,h}+\Delta W_{q,h})^{\top}W_{k,g}
    =
    M_h^{\star}.
\end{aligned}
\label{eq:app-query-min}
\end{equation}
Equivalently,
\begin{align}
    \Delta W_{q,h}^{\top}W_{k,g}
    =
    \Delta M_h.
\end{align}
Taking transposes gives
\begin{align}
    W_{k,g}^{\top}\Delta W_{q,h}
    =
    \Delta M_h^{\top}.
\end{align}

\begin{proposition}[Minimum-change query recovery]
If $W_{k,g}$ has full row rank, the minimum-Frobenius-norm solution of Eq.~\eqref{eq:app-query-min} is
\begin{align}
    \Delta W_{q,h}
    =
    (W_{k,g}W_{k,g}^{\top})^{-1}
    W_{k,g}\Delta M_h^{\top}.
\end{align}
\end{proposition}

\begin{proof}
The transposed constraint has the form
\begin{align}
    A\Delta W_{q,h}=B,
    \qquad
    A=W_{k,g}^{\top},
    \qquad
    B=\Delta M_h^{\top}.
\end{align}
The minimum-norm solution of a linear system is given by the Moore--Penrose pseudoinverse:
\begin{align}
    \Delta W_{q,h}
    =
    A^{\dagger}B.
\end{align}
Since $A=W_{k,g}^{\top}$ and $W_{k,g}$ has full row rank,
\begin{align}
    A^{\dagger}
    =
    (A^{\top}A)^{-1}A^{\top}
    =
    (W_{k,g}W_{k,g}^{\top})^{-1}W_{k,g}.
\end{align}
Substituting gives the stated update. The ridge version replaces the inverse with a numerically stable regularized inverse.
\end{proof}

The final recovered weights are
\begin{align}
    W_{q,h}^{\star}
    =
    W_{q,h}+\Delta W_{q,h},
    \qquad
    W_{k,g}^{\star}=W_{k,g}.
\end{align}
Thus, the edited QK product is realized through a query-only update while leaving shared key weights unchanged, making the edit GQA-safe.

\begin{remark}
In practice, we use the ridge-stabilized least-squares update
\begin{align}
    \Delta W_{q,h}
    =
    (W_{k,g}W_{k,g}^{\top}+\lambda I)^{-1}
    W_{k,g}\Delta M_h^{\top}.
\end{align}
where
\begin{align}
    \lambda
    =
    \epsilon \cdot \frac{1}{r}
    \mathrm{tr}(W_{k,g}W_{k,g}^{\top}).
\end{align}
For small $\lambda$, this approximates the minimum-norm recovery while improving numerical stability.
\end{remark}

\begin{remark}
For edits constructed from the QK product SVD, $\Delta M_h$ lies in the appropriate row/column subspace, so the unregularized solution can realize the product edit exactly under the full-row-rank condition. For component-specific symmetric or antisymmetric edits, the ridge update should be interpreted as the closest query-only realization.
\end{remark}

\subsection{Global SVD vs. Component-Specific Edits}
\label{app:component-mixing}

The global SVD of $M_h$ does not, in general, preserve the symmetric and antisymmetric decomposition. Since
\begin{align}
    M_h = S_h+\Omega_h,
\end{align}
the singular vectors of $M_h$ can mix energy from both channels. Therefore, damping the top singular modes of $M_h$ acts as a broad spectral intervention on the full attention-logit operator.

In contrast, Sym-only and Antisym-only edits operate on the two orthogonal channels separately. Sym-only edits the eigenmodes of $S_h$, isolating reciprocal content-similarity patterns. Antisym-only edits the singular modes of $\Omega_h$, isolating directional attention patterns. These component-specific edits are therefore diagnostic: they test which attention channel carries hallucination-relevant signal.

\subsection{Bounded Off-Target Effect}
\label{app:bounded-effect}

We next show why QK Product Steering is expected to have limited off-target effects. For each selected head,
\begin{align}
    \Delta M_h
    =
    M_h^{\star}-M_h
    =
    -\alpha\sum_{i=1}^{k}\sigma_i u_i v_i^{\top}.
\end{align}
Thus,
\begin{align}
    \mathrm{rank}(\Delta M_h)\leq k,
    \qquad
    \|\Delta M_h\|_F
    =
    |\alpha|
    \left(
    \sum_{i=1}^{k}\sigma_i^2
    \right)^{1/2}.
\end{align}
For hidden states $X$, the induced attention-logit perturbation is
\begin{align}
    \Delta L_h
    =
    X\Delta M_hX^{\top}.
\end{align}
If $\|X\|_2\leq B_X$, then
\begin{align}
    \|\Delta L_h\|_F
    \leq
    \|X\|_2^2\|\Delta M_h\|_F
    \leq
    B_X^2|\alpha|
    \left(
    \sum_{i=1}^{k}\sigma_i^2
    \right)^{1/2}.
    \label{eq:app-logit-bound}
\end{align}
Therefore, the change in attention logits is controlled by the energy of the suppressed low-rank component.

If the downstream computation from the edited attention logits to the final output logits is $L_f$-Lipschitz, then
\begin{align}
    \|\Delta z\|_2
    \leq
    L_f B_X^2|\alpha|
    \left(
    \sum_{i=1}^{k}\sigma_i^2
    \right)^{1/2}.
    \label{eq:app-output-bound}
\end{align}
This does not guarantee identical predictions on every input. Rather, it shows that the edit is localized and bounded: only a small set of attention directions in selected middle layers is modified. This provides a mathematical explanation for why broad capability benchmarks can remain close to the baseline while hallucination-sensitive generation changes more noticeably.

\subsection{Prediction Stability Under a Margin}
\label{app:margin-stability}

For classification or multiple-choice VLM tasks, the bounded-output result gives a simple stability condition. Let $z$ be the original output logits and let $y^\star=\arg\max_i z_i$. Define the prediction margin
\begin{align}
    \gamma
    =
    z_{y^\star}
    -
    \max_{j\neq y^\star}z_j .
\end{align}
Suppose the edit changes every output logit by at most $\eta$:
\begin{align}
    \|\Delta z\|_{\infty}\leq \eta.
\end{align}
If
\begin{align}
    2\eta < \gamma,
\end{align}
then the predicted answer is unchanged.

Indeed, the original top logit can decrease by at most $\eta$, while any competing logit can increase by at most $\eta$. Hence the margin can shrink by at most $2\eta$. If $2\eta<\gamma$, the original top class remains the top class. This explains why single-decision benchmarks can remain stable when examples have sufficiently large margins, while long-form captioning can still be affected through many small generation-step changes.

\begin{algorithm}[t]
\caption{QK Product Steering}
\label{alg:method}
\small
\begin{algorithmic}[1]
\Require Pretrained VLM, selected layers $\mathcal{L}$, modes $k$, damping $\alpha$, ridge scale $\epsilon$
\Ensure Edited VLM

\For{each layer $\ell \in \mathcal{L}$}
    \For{each query head $h$}
        \State Find the shared key head $g$ assigned to $h$.
        \State Extract $W_{q,h}$ and $W_{k,g}$.
        \State Compute thin SVD of the QK product:
        \[
        M_h = W_{q,h}^{\top}W_{k,g}=U\Sigma V^{\top}.
        \]
        \State Dampen top-$k$ singular values:
        \[
        \Sigma^\star_{ii} =
        \begin{cases}
        (1-\alpha)\Sigma_{ii}, & i \leq k,\\
        \Sigma_{ii}, & i > k.
        \end{cases}
        \]
        \State Form edited product:
        \[
        M_h^\star = U\Sigma^\star V^{\top},
        \qquad
        \Delta M_h = M_h^\star - M_h.
        \]
        \State Set
        \[
        \lambda =
        \epsilon \cdot \frac{1}{r}
        \mathrm{tr}(W_{k,g}W_{k,g}^{\top}).
        \]
        \State Compute query-only update:
        \[
        \Delta W_{q,h}
        =
        (W_{k,g}W_{k,g}^{\top}+\lambda I)^{-1}
        W_{k,g}\Delta M_h^{\top}.
        \]
        \State Update query weight:
        \[
        W_{q,h} \leftarrow W_{q,h}+\Delta W_{q,h}.
        \]
        \State Keep key weight fixed:
        \[
        W_{k,g} \leftarrow W_{k,g}.
        \]
    \EndFor
\EndFor

\State \Return edited VLM
\end{algorithmic}
\end{algorithm}

\section{Implementation Details}
\label{sec:implementation-details}

We provide implementation details to make the proposed edit reproducible. QK Product Steering is applied as a one-time offline modification to the pretrained model weights. For each selected layer and query head, we identify the corresponding shared key head under grouped-query attention, compute the per-head QK product through the efficient thin-SVD procedure, damp the selected spectral modes, and recover the edited query weight using the closed-form query-only update. All computations are performed directly on the attention projection weights; no training data, gradient updates, or loss optimization are used. After the edited weights are written back to the model, inference uses the original architecture, tokenizer, image processor, and decoding configuration unchanged.

\subsection{Proposed Algorithm}

Algorithm~\ref{alg:method} summarizes the full QK Product Steering procedure. The method is applied once to a pretrained VLM over a selected set of layers and query heads. For each query head, we first identify its corresponding shared key head under grouped-query attention and form the per-head QK product. We then compute its spectral decomposition, suppress the top-$k$ singular modes, and obtain an edited product. Finally, the edited product is mapped back to the model through a closed-form query-only update, while the shared key weight is kept fixed. This makes the edit GQA-safe and preserves the original inference procedure. Since all steps are computed directly from pretrained weights, the algorithm requires no training data, gradient updates, or decoding-time modification.

\subsection{Training-Free and Zero-Cost Edit}
QK Product Steering is computed directly from pretrained attention weights. It does not use training data, optimize a loss, or backpropagate through the model. Once the edited query weights are written back, the architecture and decoding procedure remain unchanged. Therefore, the edited model has the same inference cost as the original model.

\subsection{Statistical Significance}
\label{sec:statistical-significance}

We assess statistical significance using paired bootstrap resampling over image IDs. For CHAIR evaluation, each edited model is paired with the corresponding unedited baseline on the same COCO 2014 \texttt{val} images and we compute the distribution of $\Delta$CHAIR$_s$ and $\Delta$CHAIR$_i$ over paired-bootstrap samples. We use $n_{\text{boot}}{=}200$ resamples for the full-split $N{=}4{,}946$ runs (where each bootstrap pass is expensive) and $n_{\text{boot}}{=}1{,}000$ for the $N{=}1{,}000$ Qwen $k$- and $\alpha$-sweeps used in the ablations. We report $95\%$ percentile confidence intervals and the bootstrap-estimated $P(\Delta > 0)$, treating an edit as statistically reliable only when its CI excludes zero. The same paired setup is used for all CHAIR controls, including bottom-$k$, random-$k$, symmetric-only, antisymmetric-only, and amplification.

Table~\ref{tab:ci} reports paired-bootstrap CIs at the preferred middle-band setting $(k{=}3,\alpha{=}1)$ on Qwen2.5-VL-7B (Sym top $k_s{=}3$, $N{=}1{,}000$) and Pixtral-12B (QK Product Steering top-$k$, $N{=}4{,}946$). On Qwen, the symmetric-channel edit excludes zero on both $\Delta$CHAIR$_s$ and $\Delta$CHAIR$_i$, while the matched bottom-$k$ control on the same channel does not -- the signal lives in the top eigenmodes of $S_h$, not in arbitrary low-rank perturbations of $\Wq$. On Pixtral, QK Product Steering top-$k$ excludes zero on $\Delta$CHAIR$_i$ with $P(\Delta>0)=0.000$, whereas its matched random-$k$ control does not. Together these CIs back the body's main hallucination claims at the cleanest available sample size per model.

\begin{table*}[h]
\centering
\small
\setlength{\tabcolsep}{3pt}
\begin{tabular}{lrcrc}
\toprule
Variant & $\Delta_s$ & 95\% CI$_s$ & $\Delta_i$ & 95\% CI$_i$ \\
\midrule
\multicolumn{5}{l}{\emph{Qwen2.5-VL-7B} ($N{=}1{,}000$)} \\
Sym top $k_s{=}3$        & $-2.17$ & $[-2.95, -1.33]$ & $-0.67$ & $[-1.44, +0.03]$ \\
Antisym top $k_a{=}6$    & $-0.12$ & $[-0.92, +0.62]$ & $+0.31$ & $[-0.44, +1.10]$ \\
Sym bot $k_s{=}3$ (ctrl) & $+0.14$ & $[-0.56, +0.80]$ & $-0.52$ & $[-1.18, +0.13]$ \\
Antisym bot $k_a{=}6$ (ctrl) & $+0.15$ & $[-0.57, +0.86]$ & $-0.58$ & $[-1.28, +0.09]$ \\
\midrule
\multicolumn{5}{l}{\emph{Pixtral-12B} ($N{=}4{,}946$)} \\
QK Product Steering top $k{=}3$ & $-0.36$ & $[-0.85, +0.17]$ & $-0.83$ & $[-1.26, -0.33]$ \\
Random $k{=}3$ (ctrl)           & $-0.17$ & $[-0.56, +0.21]$ & $-0.35$ & $[-0.74, +0.13]$ \\
Sym top $k_s{=}3$               & $-0.06$ & $[-0.41, +0.32]$ & $-0.17$ & $[-0.56, +0.21]$ \\
Antisym top $k_a{=}6$           & $-0.10$ & $[-0.50, +0.32]$ & $-0.37$ & $[-0.73, +0.04]$ \\
\bottomrule
\end{tabular}
\caption{Paired-bootstrap $95\%$ CIs on $\Delta$CHAIR$_s$ ($\Delta_s$) and $\Delta$CHAIR$_i$ ($\Delta_i$) at the preferred middle-band $(k{=}3,\alpha{=}1)$ edit, $n_{\text{boot}}{=}200$. Bolded CIs (above) exclude zero. The Qwen Sym top-$k_s{=}3$ row defends the main-paper claim that the symmetric-channel CHAIR$_s$ drop on Qwen is CI-significant; the Pixtral QK Product Steering top-$k$ row defends the corresponding claim on Pixtral $\Delta$CHAIR$_i$. Matched bottom-$k$ and random-$k$ controls in both blocks have CIs that include zero, ruling out arbitrary low-rank perturbations of $\Wq$.}
\label{tab:ci}
\end{table*}

For probe and capability benchmarks (POPE, HallusionBench, MME, MMMU) we report raw within-pipeline differences and interpret them relative to the matched random-mode control, since these evaluations are used to verify capability preservation rather than to establish the main hallucination claim.

\subsection{Model Architectures and Details}
\label{sec:appendix-models-hardware}

\noindent\textbf{Per-model architectures.} The three GQA VLMs used in the main paper have the following backbones. Qwen/Qwen2.5-VL-7B-Instruct: $28$ transformer blocks, $28$ query heads, $4$ key/value heads (group size $7$), head dimension $128$. OpenGVLab/InternVL3-8B: $28$ transformer blocks, $28$ query heads, $7$ key/value heads (group size $4$), head dimension $128$. mistral-community/pixtral-12b (the HF-native repackage of Mistral's Pixtral-12B): $40$ transformer blocks, $32$ query heads, $8$ key/value heads (group size $4$), head dimension $128$. All three are loaded in \texttt{bfloat16} with \texttt{attn\_implementation=eager}; the eager attention path is what permits the per-(layer, head) query weight rewrites used by QK Product Steering.

\subsection{Hardware and Compute Cost}
All experiments -- the apply-edit step, the CHAIR / POPE / HallusionBench / MME / MMMU evaluations, the teacher-forced attention measurements, and the per-head spectral decomposition diagnostics -- run on a single NVIDIA A100 (80GB) GPU. Each job loads one model into memory once, applies the edit in place, and then runs the benchmark serially with greedy decoding; no multi-GPU sharding or quantization is used. The apply-edit step itself completes in $\sim 6$--$10$ seconds for all three models and contributes negligibly to wall-clock time relative to generation. Approximate per-job wall-clock times: CHAIR on the full COCO 2014 \texttt{val} split ($N{=}4{,}946$) takes $\sim 1.5$--$2.5$ hours on Qwen2.5-VL-7B and InternVL3-8B and $\sim 11$ hours on Pixtral-12B (longer due to model size); POPE ($N{=}9{,}000$) and MME ($N{\approx}2{,}374$) each take $\sim 20$--$35$ minutes per variant; HallusionBench ($N{=}447$) and MMMU val ($N{=}900$) each take $\sim 5$--$15$ minutes per variant; the per-head spectral decomposition is GPU-only weight math and finishes in under a minute per model.

\subsection{Artifact Use}
\label{sec:artifact-use}

We use existing models, datasets, and benchmarks consistently with their intended research use. The evaluated VLM checkpoints and benchmarks, including COCO, CHAIR, POPE, HallusionBench, MME, and MMMU, are used only for research evaluation of hallucination mitigation and multimodal capability. We do not introduce a new dataset or pretrained model artifact. Any released code is intended for research use, and users should comply with the licenses and access conditions of the underlying models and datasets.

\section{Additional Experimental Results}
\label{sec:appendix-additional}

This section provides implementation and evaluation details that support the main experimental claims. We first describe the per-model architectures and the hardware used to run all experiments. We then describe the benchmark protocols, sample sizes, decoding settings, and scoring procedures used for hallucination, probe, and capability evaluation. We then define the control edits used to test whether the observed gains come from dominant QK modes rather than arbitrary spectral perturbations. Finally, we report additional null checks and attention-trajectory analyses, which show that the effect is not explained by local token suppression or immediate pre-noun re-grounding, but instead emerges through downstream attention changes after the edited middle layers.

\subsection{Benchmarks and Evaluation Protocol}

\noindent\textbf{Per-benchmark protocols and sample sizes.} Generative hallucination is evaluated with CHAIR on the full COCO 2014 \texttt{val} split ($N = 4{,}946$) under the standard descriptive prompt with greedy decoding and no image-resolution cap; we report both per-sentence CHAIR$_s$ and per-mention CHAIR$_i$ and pair every edit against the unedited baseline on the same image IDs for paired-bootstrap CIs (resample counts in Appendix~\ref{sec:statistical-significance}). Probe hallucination uses POPE on $N = 9{,}000$ random / popular / adversarial yes/no items and HallusionBench on its $N = 447$ figure-paired items (covering $1{,}129$ underlying question instances aggregated to the paired figure level under the official scoring protocol), both scored with a regex extractor matching the official protocol. General capability uses MME (perception + cognition total) and MMMU val under the standard \texttt{lmms-eval} v0.7 multiple-choice protocol. POPE, MME, and MMMU all run through the \texttt{mmmu\_val\_qwen}, \texttt{pope}, and \texttt{mme} tasks of the same harness with each model's reference inference config; we exclude any task that requires a paid LLM judge to keep evaluation reproducible.

\subsection{Controls and Ablations}
\label{sec:controls}

At the default hyperparameters we additionally report a bottom-$k$ control selecting the smallest-magnitude modes, a random-$k$ control selecting modes uniformly at random, and (for QK Product Steering only) an amplify control at $\alpha < 0$ that doubles rather than removes the selected modes. Definitions of the three edit families (QK Product Steering, Sym-only, Antisym-only, Sym+Antisym) are in Section~\ref{sec:method}, Table~\ref{tab:sym-skew-variants}.

\subsection{Sym/Antisym $k$-Sweep}
\label{sec:appendix-k-sweep}

Table~\ref{tab:appendix-k-sweep} reports the sym/antisym top-$k$ and bottom-$k$ sweep on Qwen2.5-VL-7B at fixed $\alpha{=}1$ on $N{=}1{,}000$ COCO images. The symmetric top-$k$ edit peaks at $k_s{=}3$ on $\Delta$CHAIR$_s$ (CI excludes zero) and does not improve at $k_s{=}5$. The antisymmetric channel is non-monotone and weak across $k_a \in \{2, 6, 12\}$ -- no setting recovers the symmetric top-$3$ effect. Bottom-$k$ at matched rank is null on both channels (CIs include zero), supporting the conclusion that the hallucination signal lives in the top eigenmodes of $S_h$ rather than in arbitrary low-rank perturbations of $\Wq$ or in the antisymmetric channel.

\begin{table}[h]
\centering
\small
\setlength{\tabcolsep}{4pt}
\begin{tabular}{lcrr}
\toprule
Variant & $k$ & $\Delta$CHAIR$_s$ & 95\% CI \\
\midrule
Sym top   & 1  & $-2.38$ & -- \\
\textbf{Sym top}   & \textbf{3}  & $\mathbf{-2.17}$ & $\mathbf{[-2.95, -1.33]}$ \\
Sym top   & 5  & $-2.93$ & -- \\
Sym bot   & 3  & $+0.14$ & $[-0.56, +0.80]$ \\
Antisym top & 2  & $-1.41$ & -- \\
Antisym top & 6  & $-0.12$ & $[-0.92, +0.62]$ \\
Antisym top & 12 & $-1.67$ & -- \\
Antisym bot & 6  & $+0.15$ & $[-0.57, +0.86]$ \\
\bottomrule
\end{tabular}
\caption{Sym/Antisym top-$k$ and bottom-$k$ sweep on Qwen2.5-VL-7B at $\alpha{=}1$, middle-band edit, $N{=}1{,}000$ (baseline CHAIR$_s = 32.75$). Bootstrap CIs ($n_{\text{boot}}{=}1{,}000$) are computed for the four CI-reported rows; rows without CIs are point estimates from \texttt{chair\_metrics.json}.}
\label{tab:appendix-k-sweep}
\end{table}

\subsection{$\alpha$-Sweep with Confidence Intervals}
\label{sec:appendix-alpha-sweep}

Table~\ref{tab:appendix-alpha-sweep} reports the Sym top-$k_s{=}3$ $\alpha$-sweep on Qwen2.5-VL-7B, $N{=}1{,}000$, with paired-bootstrap CIs at each $\alpha$. The $\Delta$CHAIR$_s$ response is monotonically more negative in $\alpha$ on the zero-out range $[0, 1]$, with CIs excluding zero from $\alpha{=}0.5$ onward. The $\alpha{>}1$ rows continue to lower CHAIR$_s$ -- the selected modes are sign-flipped rather than removed -- but as Section~\ref{sec:appendix-capability-tradeoff} shows, this CHAIR gain comes with a measurable capability cost.

\begin{table}[h]
\centering
\small
\setlength{\tabcolsep}{4pt}
\begin{tabular}{lrr}
\toprule
$\alpha$ & $\Delta$CHAIR$_s$ & 95\% CI \\
\midrule
$0.3$ & $-0.56$ & $[-1.34, +0.15]$ \\
$0.5$ & $-1.02$ & $[-1.74, -0.21]$ \\
$0.7$ & $-1.39$ & $[-2.20, -0.61]$ \\
$1.0$ & $-2.17$ & $[-2.95, -1.33]$ \\
\bottomrule
\end{tabular}
\caption{$\alpha$-sweep for Sym top-$k_s{=}3$ on Qwen2.5-VL-7B, middle-band edit, $N{=}1{,}000$, $n_{\text{boot}}{=}1{,}000$. $\Delta$CHAIR$_s$ becomes more negative monotonically in $\alpha$; CIs exclude zero from $\alpha{=}0.5$ onward.}
\label{tab:appendix-alpha-sweep}
\end{table}

\subsection{Joint $(k, \alpha)$ Grid}
\label{sec:appendix-joint-grid}

We additionally run a 122-cell joint $(k, \alpha)$ grid on Qwen2.5-VL-7B at $N{=}500$ across all three edit families: QK Product Steering top-$k$, Sym top-$k_s$, and Antisym top-$k_a$. CHAIR$_s$ at fixed $\alpha \ge 1$ decreases monotonically with $k$ for PSVD and Sym; Antisym remains weak at every $(k, \alpha)$ cell. The per-method grid minima are reported in Table~\ref{tab:appendix-joint-grid}. The default cell $(k{=}3,\alpha{=}1)$ is not the global CHAIR$_s$ minimum: the PSVD grid minimum sits at $(k{=}6,\alpha{=}2)$ at CHAIR$_s = 24.78$ versus the baseline $31.44$. However, the larger-$k$ and larger-$\alpha$ cells carry a capability cost that the default cell does not, as quantified in Section~\ref{sec:appendix-capability-tradeoff}.

\begin{table}[h]
\centering
\small
\setlength{\tabcolsep}{4pt}
\begin{tabular}{lccc}
\toprule
Method & best $k^\star$ & best $\alpha^\star$ & CHAIR$_s$ \\
\midrule
Baseline       & --  & --  & $31.44$ \\
PSVD top-$k$   & $6$  & $2.0$ & $\mathbf{24.78}$ \\
Sym top-$k_s$  & $4$  & $2.0$ & $27.28$ \\
Antisym top-$k_a$ & $10$ & $2.0$ & $28.16$ \\
\bottomrule
\end{tabular}
\caption{Per-method grid minima from the 122-cell joint $(k, \alpha)$ sweep on Qwen2.5-VL-7B, $N{=}500$. Antisym's minimum sits well above PSVD's and Sym's at every $(k, \alpha)$ cell.}
\label{tab:appendix-joint-grid}
\end{table}

\subsection{Capability Tradeoff at Higher $(k, \alpha)$}
\label{sec:appendix-capability-tradeoff}

To isolate the capability cost of each axis, we evaluate POPE / MME / MMMU at three Qwen cells: the default $(k{=}3,\alpha{=}1)$; the matched-$k$ stronger-$\alpha$ cell $(k{=}3,\alpha{=}2)$; and the matched-$\alpha$ higher-$k$ cell $(k{=}6,\alpha{=}1)$. Table~\ref{tab:appendix-capability-tradeoff} shows the result. Both axes independently cost capability, with $\alpha$ the larger lever: MMMU loses $7.11$pp at $\alpha{=}2$ versus $3.78$pp at $k{=}6$, and MME loses $\sim 4\times$ more at $\alpha{=}2$ than at $k{=}6$. POPE F1 is essentially unchanged in all three cells. The default cell $(k{=}3,\alpha{=}1)$ is therefore the smallest principled cell where the CHAIR drop is CI-significant and capability is preserved within the random-control noise budget.

\begin{table}[h]
\centering
\small
\setlength{\tabcolsep}{3pt}
\begin{tabular}{lrrr}
\toprule
Setting & MMMU & MME & POPE F1 \\
\midrule
Baseline                       & $51.56$ & $2339.2$ & $86.18$ \\
$(k{=}3,\alpha{=}1)$ default & $50.89$ & $2344.4$ & $86.68$ \\
$(k{=}6,\alpha{=}1)$ -- $k$ effect    & $47.11$ & $2321.0$ & $87.07$ \\
$(k{=}3,\alpha{=}2)$ -- $\alpha$ effect & $43.78$ & $2253.5$ & $85.74$ \\
\bottomrule
\end{tabular}
\caption{Capability tradeoff at higher $(k, \alpha)$ on Qwen2.5-VL-7B. MMMU and MME degrade monotonically with both $k$ and $\alpha$; $\alpha$ is the larger lever. POPE F1 is essentially unchanged across the three cells.}
\label{tab:appendix-capability-tradeoff}
\end{table}

\subsection{Cross-Architecture $k$ Selection}
\label{sec:appendix-cross-arch-k}

The full $(k, \alpha)$ ablation in Sections~\ref{sec:appendix-k-sweep}--\ref{sec:appendix-capability-tradeoff} is conducted on Qwen2.5-VL-7B and drives the $(k{=}3, \alpha{=}1)$ default on that model. For the cross-architecture rows in Table~\ref{tab:main}, we keep $\alpha{=}1$ and the middle-third edit band fixed and select $k$ per model as the lower-CHAIR$_s$ cell at the same protocol: $k{=}1$ for InternVL3-8B and $k{=}3$ for Pixtral-12B. The InternVL3 choice is consistent with its higher mid-band $E_3$ profile (Figure~\ref{fig:top3_energy_per_layer}), where a smaller $k$ keeps the edit surgical.

\subsection{Per-Token and Pre-Noun Null Checks}
\label{sec:appendix-nulls}

Two natural null hypotheses for the CHAIR$_s$ reduction would be that QK Product Steering selectively suppresses the per-token log-probability of hallucinated noun tokens at fixed prefix, or that it visibly re-grounds attention at the position immediately before a hallucinated noun. Under teacher-forced prefixes neither holds. QK Product Steering top-$3$ leaves the hallucinated-token log-probability essentially unchanged ($-0.004$) while reducing the grounded-token log-probability \emph{more} ($-0.049$); the attention-to-image probability at the pre-noun position is within $0.002$ of every control (Table~\ref{tab:mechanism}). The mechanism is therefore neither token-local nor position-local; the rest of the discussion in the main text locates it.

\begin{table}[h]
  \centering\small
  \setlength{\tabcolsep}{3pt}
  \begin{tabular}{lcccc}
    \toprule
     & \multicolumn{2}{c}{$\Delta \log P$} & \multicolumn{2}{c}{attn$\to$img} \\
    \cmidrule(lr){2-3}\cmidrule(lr){4-5}
    Variant & hallu & gr. & hallu & gr. \\
    \midrule
    Matched-norm rand.  & $-0.015$ & $+0.006$ & 0.126 & 0.131 \\
    Random-$k$           & $+0.003$ & $+0.013$ & 0.126 & 0.131 \\
    QK-Prod top-$3$, $\alpha{=}1$ & $-0.004$ & $-0.049$ & 0.125 & 0.129 \\
    \bottomrule
  \end{tabular}
  \caption{Per-token log-probability deltas and attention-to-image probability at the pre-noun position, split by hallucinated vs.\ grounded noun tokens. $N{=}291$ hallucinated noun tokens and $N{=}214$ grounded noun tokens across 100 COCO images, middle layers.}
  \label{tab:mechanism}
\end{table}

\subsection{Per-Layer Attention Trajectory}
\label{sec:appendix-attn-trajectory}

\begin{figure}[h]
  \centering
  \includegraphics[width=1.0\textwidth]{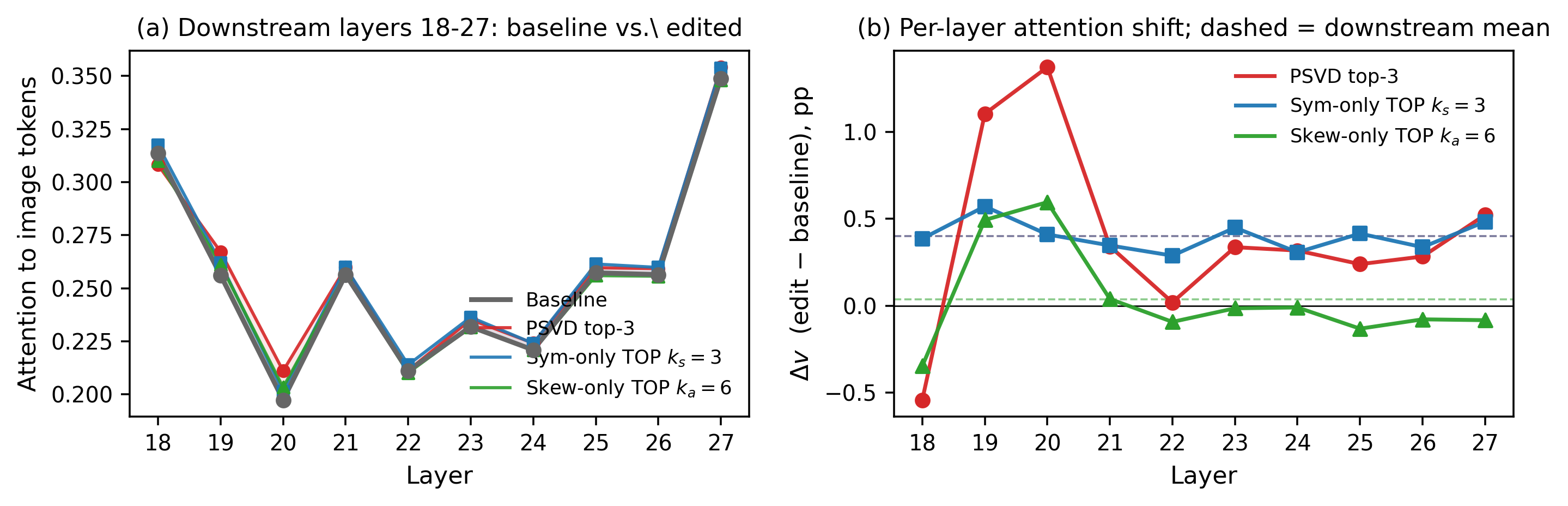}
  \caption{Per-layer shift $\Delta v$ (edit $-$ baseline, percentage points) across all $28$ layers of Qwen2.5-VL-7B for the three spectral methods, teacher-forced. Shaded band: edited middle layers $L_9$--$L_{17}$. The two CHAIR-active methods (QK Product Steering, Sym-only) produce a uniform positive shift across every downstream unedited layer; Antisym-only, which is inert on CHAIR, does not.}
  \label{fig:attention_delta}
\end{figure}

Figure~\ref{fig:per_head_dv} further supports the trajectory-cascade explanation in Section~\ref{sec:trajectory}. The edited layers show nearly zero layer-mean shift due to per-head cancellation, while the effect emerges in the downstream unedited layers, with the strongest increase immediately after the edit site.


\begin{figure}[h]
  \centering
  \includegraphics[width=1.0\textwidth]{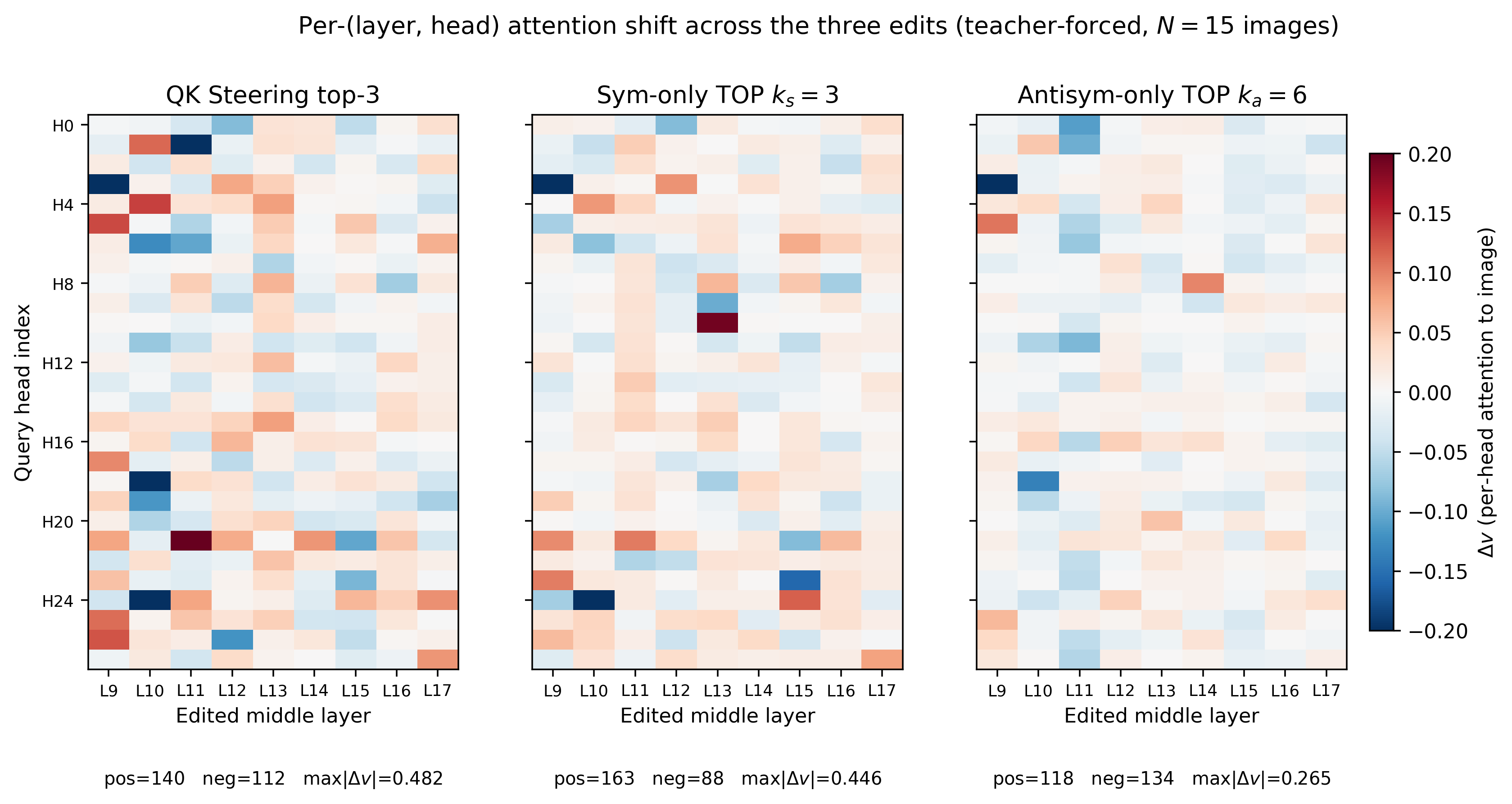}
  \caption{Per-(layer, head) attention-to-image shift $\Delta v$ inside the edited middle band ($9 \times 28 = 252$ cells), QK Product Steering top-$3$ on Qwen2.5-VL-7B. Per-head shifts reach $\pm 0.48$pp with $138$ positive and $114$ negative cells; the layer mean is $+0.019$pp.}
  \label{fig:per_head_dv}
\end{figure}

Figure~\ref{fig:per_head_dv} disaggregates the layer-mean $\Delta v$ inside the edited band over its $252$ (layer, head) cells. The per-head shifts are large in magnitude ($\pm 0.48$pp) and roughly balanced in sign ($138$ positive, $114$ negative), so they cancel in the layer mean ($+0.019$pp). They do not cancel in the residual stream those heads jointly write, which is what the downstream layers in Figure~\ref{fig:attention_delta} respond to with the uniform positive shift.

\section{Qualitative Examples with COCO Images}
\label{sec:appendix-qual}

This section provides qualitative examples illustrating how QK Product Steering changes generated captions on COCO images. We show cases where the edited model removes hallucinated object mentions from the baseline caption while preserving grounded visual content. These examples are intended to complement the quantitative CHAIR results by showing that the reduction in hallucination corresponds to more image-grounded descriptions rather than shorter or less informative captions.


Across the full COCO validation set ($N{=}4{,}946$), we identify $756$ images where QK Product Steering top-$3$ satisfies a strict improvement criterion: it removes at least one hallucinated COCO object from the baseline caption, introduces no new hallucinated objects, and preserves all grounded object mentions. Figure~\ref{tab:qualitative_appendix} presents eight representative examples, ranked by the number of hallucinations removed. For each case, we show the COCO image together with truncated baseline and QK Product Steering top-$3$ captions.

\section{Limitations}
\label{sec:limitations-future}




QK Product Steering is a lightweight post-hoc edit with a focused scope: it modifies the content-based query-key product while leaving the architecture and inference procedure unchanged. This makes it efficient, but it does not explicitly capture position-dependent effects such as rotary positional embeddings. Extending the method to position-aware query-key interactions is an important future direction.
The method also depends on a few hyperparameters, including edited layers, heads, number of modes, and damping strength. Although middle-layer dominant modes work well across multiple VLMs in our experiments, the best configuration may vary by architecture. 


\section{Ethics Statement}
\label{sec:appendix-ethics}

QK Product Steering is a lightweight post-hoc intervention for reducing object hallucination in VLM-generated captions. It requires no additional data, annotations, gradient updates, or decoding-time changes, making it easy to apply to existing models.
However, QK Product Steering is not a universal hallucination solution. It mainly targets free-form object hallucination by weakening dominant prior-driven attention modes, and may not address failures from weak visual perception, ambiguous images, or adversarial yes/no reasoning. Therefore, in safety-critical settings, QK Product Steering should be used as a complementary reliability tool alongside stronger grounding, verification, and human oversight.

\section{Use of Large Language Models (LLMs)}

Large language models (LLMs) were used only as writing aids during manuscript preparation. Their use was limited to improving grammar, clarity, wording, and organization. The authors reviewed and revised all LLM-assisted text to ensure that the final manuscript accurately represents the intended meaning, technical claims, and scientific contributions.

%
%
%
%

\providecommand{\Method}{QK Product Steering}

\clearpage
\onecolumn

\providecommand{\HALLU}[1]{\textcolor{red!75!black}{\textbf{#1}}}
\providecommand{\GROUND}[1]{\textcolor{green!50!black}{\textbf{#1}}}
\providecommand{\RECOV}[1]{\textcolor{blue!70!black}{\textbf{#1}}}

\subsection*{Color Reference}
\HALLU{Red} = hallucinated by the baseline, \GROUND{green} = grounded, \RECOV{blue} = grounded object the baseline missed and the edit recovers.

\noindent\begin{minipage}{\textwidth}
\paragraph{Example 1 (image 120853): plated sandwich with coleslaw.}
\textsc{Removed:} \emph{fork}; \textsc{Kept:} \emph{bowl, carrot, cell phone, dining table, sandwich}.

\noindent
\begin{minipage}[t]{0.32\textwidth}
  \vspace{0pt}
  \includegraphics[width=\linewidth]{assets/coco_verified/120853.jpg}
\end{minipage}\hfill
\begin{minipage}[t]{0.65\textwidth}
  \vspace{0pt}
  \footnotesize
  \textbf{Baseline.} This image shows a meal served on a green plate at an outdoor dining setting, likely a cafe or restaurant. The main dish is a \GROUND{sandwich} made with toasted bread, which appears to be filled with grilled chicken and possibly some vegetables. The \GROUND{sandwich} is cut diagonally into two halves, with one half placed on top of the other. Accompanying the \GROUND{sandwich} is a small white \GROUND{bowl} containing coleslaw, which includes shredded cabbage, \GROUND{carrots}, and possibly other vegetables, dressed with a light sauce. A small white plastic \HALLU{fork} is placed inside the \GROUND{bowl} for serving. In the background, there is a glass of red liquid, possibly iced tea or juice, and a \GROUND{smartphone} resting on the \GROUND{table}. A salt shaker is also visible on the left side of the frame. The \GROUND{table} has a metal mesh design, and the overall ambiance suggests a casual dining experience.

  \vspace{2pt}
  \textbf{\Method.} This image shows a meal served on a \GROUND{table}, likely at an outdoor dining area given the metal mesh \GROUND{table} and the presence of sunlight. The main focus is a plate with a \GROUND{sandwich} cut in half, revealing its contents. The \GROUND{sandwich} appears to be made with toasted bread, possibly sourdough or a similar type, and contains slices of grilled chicken as the filling.  Next to the \GROUND{sandwich}, there is a small \GROUND{bowl} filled with coleslaw, which includes shredded cabbage, \GROUND{carrots}, and other vegetables, dressed in a creamy dressing. The coleslaw adds a fresh and colorful element to the meal. In the background, there's a glass of what looks like iced tea, a \GROUND{smartphone} resting on the \GROUND{table}, and a pepper shaker. The overall setting suggests a casual dining experience, perhaps at a café or restaurant. The lighting indicates it might be daytime, and the scene has a relaxed atmosphere.
\end{minipage}
\end{minipage}

\vspace{10pt}

\noindent\begin{minipage}{\textwidth}
\paragraph{Example 2 (image 150417): birthday cake in a home kitchen.}
\textsc{Removed:} \emph{toaster}; \textsc{Kept:} \emph{bowl, cake, cup, dining table, person, refrigerator}.

\noindent
\begin{minipage}[t]{0.32\textwidth}
  \vspace{0pt}
  \includegraphics[width=\linewidth]{assets/coco_verified/150417.jpg}
\end{minipage}\hfill
\begin{minipage}[t]{0.65\textwidth}
  \vspace{0pt}
  \footnotesize
  \textbf{Baseline.} The image shows an indoor setting, likely a kitchen or dining area, where two individuals are present. The focus is on the older \GROUND{woman} seated at a \GROUND{table}, wearing glasses and a red jacket over a blue top. She appears to be looking at a small \GROUND{cake} placed in front of her on the \GROUND{table}. The \GROUND{cake} has multiple lit candles, suggesting it might be for a birthday celebration. The \GROUND{table} is cluttered with various items, including a pink \GROUND{cup}, a clear \GROUND{bowl}, some papers, and what looks like a decorative rabbit figurine. A \GROUND{man}'s arm is visible on the right side of the frame, indicating he is standing next to the \GROUND{table}. He seems to be interacting with the \GROUND{cake} or the \GROUND{woman}, possibly lighting another candle or assisting with the celebration. In the background, there is a \GROUND{refrigerator} adorned with magnets, notes, and pictures, as well as a countertop with various kitchen items such as \GROUND{containers}, utensils, and a \HALLU{toaster}. The overall atmosphere suggests a casual and warm family gathering.

  \vspace{2pt}
  \textbf{\Method.} The image shows an indoor setting, likely a kitchen or dining area, where two individuals are present. The focus is on an elderly \GROUND{woman} seated at a \GROUND{table}, wearing glasses and a red cardigan over a blue top. She appears to be looking at the \GROUND{cake} in front of her with a thoughtful expression. The \GROUND{cake} is round, white, and decorated with colorful candles, suggesting it might be for a birthday celebration. A hand, presumably belonging to another \GROUND{person}, is seen lighting one of the candles. The individual lighting the candle is partially visible, with only their arm and part of their torso shown. The \GROUND{table} has various items on it, including a pink \GROUND{cup}, a clear \GROUND{bowl}, and some papers. In the background, there's a \GROUND{refrigerator} adorned with magnets and notes, along with other kitchen appliances and cabinets. The room is lit by warm artificial light, giving the scene a cozy, homey feel. The overall atmosphere suggests a casual and intimate gathering, possibly a family celebration.
\end{minipage}
\end{minipage}

\vspace{10pt}

\noindent\begin{minipage}{\textwidth}
\paragraph{Example 3 (image 237517): patio dinner on a wooden table.}
\textsc{Removed:} \emph{bench}; \textsc{Recovered:} \emph{spoon}; \textsc{Kept:} \emph{bowl, dining table, fork, knife, potted plant}.

\noindent
\begin{minipage}[t]{0.32\textwidth}
  \vspace{0pt}
  \includegraphics[width=\linewidth]{assets/coco_verified/237517.jpg}
\end{minipage}\hfill
\begin{minipage}[t]{0.65\textwidth}
  \vspace{0pt}
  \footnotesize
  \textbf{Baseline.} This image depicts an outdoor dining setup on a wooden \GROUND{table}, likely in a backyard or garden area. The \GROUND{table} is set with various dishes and utensils, suggesting a meal is either being enjoyed or prepared. 1. Foreground:    - A white plate holds a piece of steak, roasted potatoes, and grilled zucchini slices.    - A \GROUND{knife} and \GROUND{fork} rest on the plate, indicating that someone is eating or has been eating.    - To the right of the plate, there's a small \GROUND{bowl} containing what appears to be a salad with greens and possibly some crumbled cheese or dressing. 2. Midground:    - Another plate further back contains a baked dish, possibly a quiche or pie, served with a side of bread or toast.    - Next to it, there’s a small \GROUND{bowl} filled with fresh greens, possibly a salad or herbs. 3. Background:    - The setting includes a white wooden \HALLU{bench} and a \GROUND{potted plant} with succulent-like leaves on the left side of the \GROUND{table}.    - Beyond the \GROUND{table}, there is a brick patio leading to a grassy area with a hedge in the distance, suggesting a private garden space. 4. Atmosphere:    - The scene conveys a relaxed, casual dining experience, likely during the daytime given the natural light.    - The presence of plants and the outdoor setting adds a tranquil and cozy ambiance to the meal. The overall composition suggests a simple yet enjoyable meal in a peaceful outdoor environment.

  \vspace{2pt}
  \textbf{\Method.} This image depicts an outdoor dining setup on a wooden \GROUND{table}, likely in a backyard or patio area. The \GROUND{table} is set with various dishes and utensils, suggesting a meal is being enjoyed. 1. Foreground Plate:     - A white plate holds a piece of steak, a baked potato, and two slices of grilled zucchini.    - A \GROUND{knife} and \GROUND{fork} rest on the plate, indicating that the meal is ready to be eaten. 2. Background Plate:    - Another plate contains what appears to be a baked dish, possibly a quiche or pie, served in a small bread \GROUND{bowl}.    - A \RECOV{spoon} is placed inside the dish, and it seems to have a creamy filling. 3. \GROUND{Bowl} of Salad:    - Two \GROUND{bowls} are visible on the right side of the \GROUND{table}. One contains a green salad with leafy greens and possibly some croutons or cheese.    - The other \GROUND{bowl} has a similar salad but with a different mix of greens and toppings. 4. Decorative Elements:    - To the left of the \GROUND{table}, there is a \GROUND{potted plant} with succulent-like leaves, adding a touch of greenery to the scene.    - A small pot with pink flowers is also visible, enhancing the aesthetic appeal of the setting. 5. Outdoor Setting:    - The \GROUND{table} is surrounded by a white wooden railing, which suggests the dining area is part of a deck or patio.    - Beyond the railing, there is a grassy area and a brick pathway leading into the garden.    - The background includes a hedge and some trees, indicating a private, well-maintained outdoor space. The overall ambiance of the image is cozy and inviting, with a focus on fresh food and a pleasant outdoor environment.
\end{minipage}
\end{minipage}

\vspace{10pt}

\noindent\begin{minipage}{\textwidth}
\paragraph{Example 4 (image 384661): stove with a utensil cart.}
\textsc{Removed:} \emph{scissors}; \textsc{Recovered:} \emph{refrigerator}; \textsc{Kept:} \emph{bottle, oven}.

\noindent
\begin{minipage}[t]{0.32\textwidth}
  \vspace{0pt}
  \includegraphics[width=\linewidth]{assets/coco_verified/384661.jpg}
\end{minipage}\hfill
\begin{minipage}[t]{0.65\textwidth}
  \vspace{0pt}
  \footnotesize
  \textbf{Baseline.} The image shows a kitchen setup with a stainless steel gas \GROUND{stove} and an adjacent metal shelving unit. The \GROUND{stove} has four burners on top, a control panel with digital buttons, and a storage compartment below for warming or storing items. The \GROUND{oven} door is closed, and the brand name "Bake Warm Store" is visible at the bottom of the \GROUND{oven}. To the right of the \GROUND{stove}, there is a metal shelving unit with three levels. On the top shelf, there is a large silver pot, a smaller pot, and some kitchen utensils such as spatulas and a pair of \HALLU{scissors}. There is also a \GROUND{bottle} of oil or sauce near the pots. The middle shelf holds a small red appliance (possibly a blender or food processor) and a white kettle. The bottom shelf appears to be empty. The background wall is painted yellow, and the floor is tiled with terracotta-colored tiles. The overall setting suggests a compact kitchen space with efficient use of vertical space.

  \vspace{2pt}
  \textbf{\Method.} The image shows a kitchen setup with a stainless steel gas \GROUND{stove} on the left and a small rolling cart to its right. Here's a detailed description: 1. \GROUND{Stove}:    - The \GROUND{stove} is a gas model with four burners on top.    - It has a digital control panel at the front, which appears to be displaying settings or times.    - Below the burners is an \GROUND{oven} with a black glass door, indicating it is a combination \GROUND{oven} (with both a conventional \GROUND{oven} and a broiler).    - The \GROUND{stove} has a sleek, modern design with a brushed metal finish. 2. Rolling Cart:    - The cart is made of metal and has three shelves.    - On the top shelf:      - There is a large silver pot, likely used for cooking or boiling water.      - Next to the pot, there is a \GROUND{bottle} that could contain oil or another cooking ingredient.    - On the middle shelf:      - A small red appliance, possibly a blender or food processor, is visible.      - A white kettle or teapot is also present.    - On the bottom shelf:      - A round object, possibly a lid or a lid cover, is placed on the shelf.      - A white corded device, potentially a small appliance like a food processor or mixer, is partially visible. 3. Background:    - The wall behind the \GROUND{stove} is painted yellow, adding warmth to the scene.    - To the right of the cart, part of a \RECOV{refrigerator} is visible, suggesting this is a compact kitchen space. 4. Flooring:    - The floor appears to be tiled, with a terracotta or similar color, complementing the warm tones of the room. This setup suggests a functional and organized kitchen area, with appliances and tools neatly arranged for easy access.
\end{minipage}
\end{minipage}



\noindent\begin{minipage}{\textwidth}
\paragraph{Example 5 (image 571893): kitchen pass-through with cookbooks.}
\textsc{Removed:} \emph{potted plant}; \textsc{Recovered:} \emph{vase}; \textsc{Kept:} \emph{book, oven}.

\noindent
\begin{minipage}[t]{0.32\textwidth}
  \vspace{0pt}
  \includegraphics[width=\linewidth]{assets/coco_verified/571893.jpg}
\end{minipage}\hfill
\begin{minipage}[t]{0.65\textwidth}
  \vspace{0pt}
  \footnotesize
  \textbf{Baseline.} This image shows a cozy and well-organized corner of a room, likely a kitchen or study area. The focal point is a built-in shelf with a dark marble-like finish, which holds an assortment of items. On the left side of the shelf, there's a black kettle on what appears to be a \GROUND{stovetop}, suggesting that this is part of a kitchen setup.  On the right side of the shelf, there is a collection of \GROUND{books} neatly arranged vertically. The spines of the \GROUND{books} display various colors and designs, indicating a diverse range of topics or genres. Next to the \GROUND{books}, there is a small \HALLU{potted plant} with vibrant pink flowers, adding a touch of nature and color to the scene. Above the shelf, a chandelier hangs from the ceiling, providing warm lighting for the space. To the far right, there is a decorative item resembling an evil eye amulet, often used as a talisman for protection. The walls are painted a soft yellow, complementing the overall warm and inviting atmosphere of the room. The reflection on the shiny surface below the shelf adds depth to the image, creating a mirrored effect that enhances the visual appeal.

  \vspace{2pt}
  \textbf{\Method.} This image shows a cozy and well-decorated corner of a room, likely a kitchen or dining area. Here's a detailed description: 1. Foreground (Mirror Reflection):    - The left side of the image features a mirror reflecting part of the scene.    - A black kettle is visible on what appears to be a \GROUND{stovetop} or countertop.    - A small black object, possibly a kitchen appliance or utensil, is placed near the kettle. 2. Background (Window Shelf):    - A window with a white frame is positioned above the mirror reflection.    - On the windowsill, there is an organized collection of \GROUND{books} arranged vertically. The spines of the \GROUND{books} are visible, suggesting they might be cookbooks or reference \GROUND{books} due to their size and color.    - To the right of the \GROUND{books}, there is a decorative \RECOV{vase} holding pink flowers, adding a touch of color and freshness to the space.    - A small chandelier hangs from the ceiling, providing soft lighting for the area below.    - A hanging ornament resembling an evil eye amulet is attached to the wall to the right of the window, adding a cultural or decorative element. 3. Wall and Color Scheme:    - The walls are painted a warm yellow tone, which gives the space a welcoming and inviting feel.    - The window frame and trim are white, creating a contrast against the yellow wall. 4. Lighting:    - The lighting in the room is soft and warm, likely coming from the chandelier and possibly additional light sources not visible in the image. Overall, the image conveys a sense of warmth, functionality, and personalization, with a mix of practical elements like the kettle and \GROUND{books}, and decorative touches like the flowers and evil eye ornament.
\end{minipage}
\end{minipage}

\vspace{10pt}

\noindent\begin{minipage}{\textwidth}
\paragraph{Example 6 (image 58350): two people on a couch eating pizza.}
\textsc{Removed:} \emph{chair}; \textsc{Recovered:} \emph{couch}; \textsc{Kept:} \emph{person, pizza}.

\noindent
\begin{minipage}[t]{0.32\textwidth}
  \vspace{0pt}
  \includegraphics[width=\linewidth]{assets/coco_verified/58350.jpg}
\end{minipage}\hfill
\begin{minipage}[t]{0.65\textwidth}
  \vspace{0pt}
  \footnotesize
  \textbf{Baseline.} The image shows two women sitting indoors, likely in a casual setting such as a living room or a similar space. The \GROUND{woman} in the foreground is seated on a patterned \HALLU{chair} with a red and beige design. She has blonde hair styled up and is wearing a black top with a light-colored strap visible underneath. She appears to be eating a slice of \GROUND{pizza}, which she holds close to her mouth. Her expression seems neutral or slightly focused. The \GROUND{woman} in the background is wearing a green top and is partially out of focus. She is holding what looks like a yellow object, possibly a utensil or part of a game, and is looking toward the camera with a slight smile. The background is plain and light-colored, suggesting an indoor environment with minimal decoration. The overall atmosphere of the photo feels informal and relaxed.

  \vspace{2pt}
  \textbf{\Method.} The image shows two women sitting on a \RECOV{couch} in what appears to be an indoor setting, possibly a living room or a casual gathering space. The \GROUND{woman} in the foreground is holding a large slice of \GROUND{pizza} with a visible topping that looks like cheese and possibly tomato sauce. She has blonde hair styled up and is wearing a black top with a light-colored strap visible underneath. Her expression seems neutral as she holds the \GROUND{pizza}. The \GROUND{woman} in the background is wearing a green top and is partially out of focus. She appears to be holding a yellow object, possibly a utensil or tool, and her expression suggests she might be engaged in conversation or reacting to something happening off-camera. The \RECOV{couch} they are sitting on has a patterned design with a mix of red and beige tones. The overall atmosphere of the image feels casual and informal.
\end{minipage}
\end{minipage}

\vspace{10pt}

\noindent\begin{minipage}{\textwidth}
\paragraph{Example 7 (image 470924): outdoor pizza lunch.}
\textsc{Removed:} \emph{umbrella}; \textsc{Recovered:} \emph{fork, knife}; \textsc{Kept:} \emph{car, chair, dining table, person, pizza}.

\noindent
\begin{minipage}[t]{0.32\textwidth}
  \vspace{0pt}
  \includegraphics[width=\linewidth]{assets/coco_verified/470924.jpg}
\end{minipage}\hfill
\begin{minipage}[t]{0.65\textwidth}
  \vspace{0pt}
  \footnotesize
  \textbf{Baseline.} The image shows a group of four \GROUND{people} sitting at an outdoor \GROUND{dining table}, enjoying a meal together. The setting appears to be a casual restaurant or café with wooden \GROUND{tables} and \GROUND{chairs}. The \GROUND{table} is set with plates of \GROUND{pizza}, glasses of beer, and cutlery. The \GROUND{pizzas} look freshly made, with visible toppings such as cheese, herbs, and possibly vegetables. The individuals are dressed casually, with one \GROUND{person} wearing a red jacket, another in a black hoodie, and the others in darker clothing. They seem relaxed and engaged in conversation, with some looking directly at the camera while others appear to be mid-conversation. The background reveals other diners seated at similar \GROUND{tables}, and the area is shaded by large \HALLU{umbrellas}, suggesting it's a sunny day. Beyond the \HALLU{umbrellas}, there is a street view with parked \GROUND{cars} and a sidewalk, indicating that the location is likely in an urban or suburban area. The atmosphere seems lively yet laid-back, typical of a social gathering over food.

  \vspace{2pt}
  \textbf{\Method.} This image shows four individuals seated at an outdoor \GROUND{dining table}, enjoying a meal together. The setting appears to be a casual restaurant or café with other patrons and \GROUND{tables} visible in the background. The \GROUND{table} is set with plates of \GROUND{pizza}, glasses of beer, and cutlery. The individuals seem relaxed and engaged in conversation. - Foreground:    - On the left, there's a \GROUND{person} wearing a red jacket and sunglasses, holding a \RECOV{fork} and \RECOV{knife}, seemingly about to eat.   - In the center, another individual is wearing a black hoodie and sunglasses, looking directly at the camera with a slight smile.   - To the right, a \GROUND{person} in a dark hoodie is also looking at the camera, appearing to be in mid-conversation. - Background:   - Other diners can be seen in the background, some seated at separate \GROUND{tables}, while others are standing or walking around.   - The outdoor seating area has wooden \GROUND{tables} and \GROUND{chairs}, and the ambiance suggests a sunny day with natural light illuminating the scene. - Food and Drinks:   - The \GROUND{table} is laden with several \GROUND{pizzas}, each with different toppings that appear to include cheese, vegetables, and possibly meat.   - There are multiple glasses of beer on the \GROUND{table}, indicating a casual dining experience. - Environment:   - The location seems to be a street-side café or restaurant with a view of a road and parked \GROUND{cars} in the distance.   - The weather looks pleasant, as suggested by the bright lighting and shadows cast by the \GROUND{people} and objects. The overall mood of the image is friendly and convivial, capturing a moment of shared enjoyment among friends or acquaintances.
\end{minipage}
\end{minipage}

\vspace{10pt}

\noindent\begin{minipage}{\textwidth}
\paragraph{Example 8 (image 340175): living and dining room with a bookshelf.}
\textsc{Recovered:} \emph{bench, potted plant}; \textsc{Kept:} \emph{book, chair, couch, dining table}.

\noindent
\begin{minipage}[t]{0.32\textwidth}
  \vspace{0pt}
  \includegraphics[width=\linewidth]{assets/coco_verified/340175.jpg}
\end{minipage}\hfill
\begin{minipage}[t]{0.65\textwidth}
  \vspace{0pt}
  \footnotesize
  \textbf{Baseline.} This image depicts a cozy and well-lit living room with a mix of modern and rustic elements. The room features a large window that allows natural light to flood the space, offering a view of trees outside. The walls are painted a neutral gray tone, which complements the wooden flooring. In the center of the room is a wooden \GROUND{dining table} with a simple design, accompanied by red \GROUND{chairs} with a modern aesthetic. The \GROUND{table} appears to be made from reclaimed wood, adding a touch of rustic charm. A small decorative item, possibly a candle or ornament, sits on the \GROUND{table} near the window. To the right of the dining area is a dark brown leather \GROUND{sofa}, positioned against the wall. In front of the \GROUND{sofa} lies a cowhide rug, adding texture and warmth to the space. A matching ottoman sits next to the \GROUND{sofa}, providing additional seating or a place to rest one's feet. On the left side of the room, there is a tall bookshelf filled with \GROUND{books}, suggesting an interest in reading or collecting literature. The bookshelf has a warm, wooden finish that matches the flooring. Above it, a decorative chandelier with a circular, flower-like design hangs from the ceiling, adding a unique and artistic touch to the room. The overall atmosphere of the room is inviting and comfortable, blending functionality with style. The combination of natural materials like wood and leather, along with the modern furniture and decor, creates a harmonious and stylish living space.

  \vspace{2pt}
  \textbf{\Method.} This image depicts a cozy and well-lit living room with a mix of modern and rustic elements. Here's a detailed description: 1. Furniture and Layout:    - The room features a wooden parquet floor, which adds warmth to the space.    - A dark red leather \GROUND{sofa} is positioned against the wall on the right side, accompanied by a small white ottoman and a minimalist \GROUND{chair} with a curved metal frame.    - On the left side, there is a wooden \GROUND{dining table} with a simple design, surrounded by bright red \GROUND{chairs} that contrast with the neutral tones of the room. 2. Lighting and Decor:    - Two large, decorative chandeliers hang from the ceiling, adding an artistic touch to the room.    - Natural light streams in through a large window in the center of the back wall, which offers a view of trees outside, suggesting a peaceful setting.    - The window is framed by a radiator below it, and the wall above the radiator has a few framed pictures or artwork. 3. Decorative Elements:    - A bookshelf filled with \GROUND{books} occupies the far left corner, indicating a love for reading or a scholarly interest.    - A cowhide rug lies on the floor near the \GROUND{sofa}, adding texture and a natural element to the room.    - A small \RECOV{potted plant} sits on the windowsill, bringing a touch of greenery indoors.    - A colorful abstract painting hangs on the wall to the left of the window, adding vibrancy to the otherwise neutral palette. 4. Other Features:    - A wooden \RECOV{bench} with a simple design is placed near the \GROUND{dining table}, offering additional seating.    - The walls are painted in a muted gray tone, which complements the warm wood tones and red accents.    - The overall aesthetic combines functionality with a touch of artistic flair, creating a comfortable and inviting atmosphere.
\end{minipage}
\end{minipage}

\vspace{10pt}

\label{tab:qualitative_appendix}

\end{document}